\relax
\RequirePackage{etex}
\documentclass[a4paper, 11pt]{article} 
\usepackage[margin=1in]{geometry}
\usepackage[hyphens]{url}
\usepackage{graphicx}
\usepackage{caption}
\usepackage{titlesec}
\titleformat{\section}{\normalfont\bfseries\large\raggedright}{\thesection.\ }{0.10cm}{}
\titlespacing*{\section}{0pt}{0.24in plus .8ex}{0.10in plus .2ex}
\titleformat{\subsection}{\normalfont\bfseries\normalsize\raggedright}{\thesubsection.\ }{0.08cm}{}
\titlespacing*{\subsection}{0pt}{0.2in plus .8ex}{0.08in plus .2ex}

\usepackage{natbib}
\bibpunct{(}{)}{;}{a}{,}{,}

\usepackage{hyperref}
\hypersetup{colorlinks,linkcolor=blue,citecolor=blue,
    urlcolor=magenta,linktocpage,plainpages=false}

\usepackage{amssymb,amsmath,amsthm,microtype,booktabs,enumitem,tikz,subcaption,mathtools,array,ifthen,pgf,bbm,comment}
\usepackage[export]{adjustbox}
\usepackage[utf8]{inputenc}
\usetikzlibrary{arrows,positioning,backgrounds,automata}

\newcommand{\reals}{\mathbb{R}}
\newcommand{\X}{\mathbb{X}}
\newcommand{\Y}{\mathbb{Y}}
\newcommand{\Z}{\mathbb{Z}}
\renewcommand{\P}{\mathbb{P}}

\newcommand{\x}{\textbf{x}}
\newcommand{\y}{\textbf{y}}
\newcommand{\z}{\textbf{z}}
\newcommand{\p}{\textbf{p}}

\newcommand{\h}{\textbf{h}}

\newcommand{\0}{\textbf{0}}

\renewcommand{\bar}{\overline}
\newcommand{\Trans}{\intercal}
\DeclareMathOperator*{\relintnp}{relint}
\newcommand{\relint}[1]{\operatorname{\relintnp}(#1)}
\DeclareMathOperator*{\argmax}{arg\,max}
\newcommand\restr[2]{{
  \left.\kern-\nulldelimiterspace 
  #1 
  \vphantom{\big|} 
  \right|_{#2} 
  }}

\newtheorem{definition}{Definition}
\newtheorem{theorem}{Theorem}

\title{Learning in Matrix Games can be Arbitrarily Complex}

\author{Gabriel P. Andrade\textsuperscript{\rm 1} \and Rafael Frongillo\textsuperscript{\rm 1} \and Georgios Piliouras\textsuperscript{\rm 2}\\}
\date{\textsuperscript{\rm 1} Department of Computer Science\\
University of Colorado Boulder\\
\{gabriel.andrade~;~raf\}@colorado.edu\\

\vskip8pt

\textsuperscript{\rm 2} Engineering Systems and Design\\
Singapore University of Technology and Design\\
georgios@sutd.edu.sg}

\begin{document}

\maketitle

\begin{abstract}
A growing number of machine learning architectures, such as Generative Adversarial Networks, rely on the design of games which implement a desired functionality via a Nash equilibrium. 
In practice these games have an implicit complexity (e.g.~from underlying datasets and the deep networks used) that makes directly computing a Nash equilibrium impractical or impossible.
For this reason, numerous learning algorithms have been developed with the goal of iteratively converging to a Nash equilibrium.
Unfortunately, the dynamics generated by the learning process can be very intricate and instances of training failure hard to interpret.
In this paper we show that, in a strong sense, this dynamic complexity is inherent to games.
Specifically, we prove that replicator dynamics, the continuous-time analogue of Multiplicative Weights Update,
even when applied in a very restricted class of games---known as finite matrix games---is rich enough to be able to approximate arbitrary dynamical systems.
Our results are positive in the sense that they show the nearly boundless dynamic modelling capabilities of current machine learning practices, but also negative in implying that these capabilities may come at the cost of interpretability. As a concrete example, we show how replicator dynamics can effectively reproduce the well-known strange attractor of Lonrenz dynamics (the ``butterfly effect",  Fig~\ref{fig:lorenz}) while achieving no regret.
\end{abstract}

\hspace{5pt}

\begin{figure}[!ht]
  \centering
  \includegraphics[scale = .12]{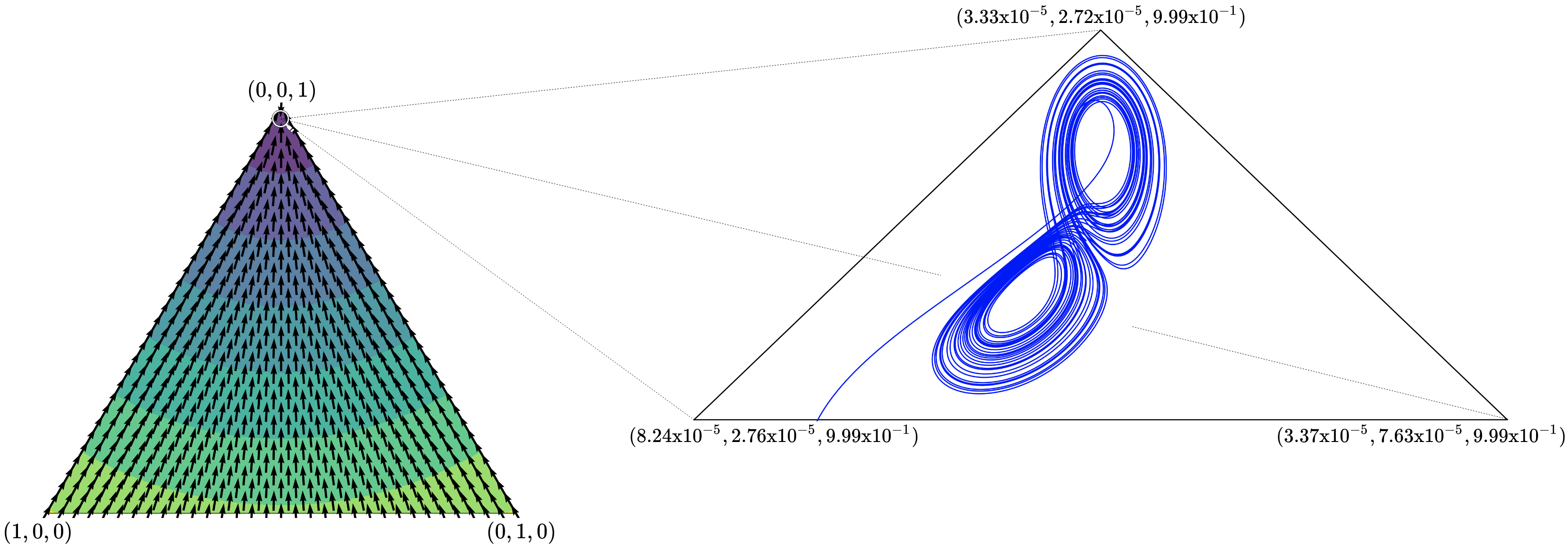}
  \caption{A no-regret strange attractor.
  The Lorenz system can be embedded in replicator dynamics on a finite matrix game.
  For more details see Sections~\S\ref{sec:lorenz} and Appendix~\ref{append:Lorenz}.}
  \label{fig:lorenz}
\end{figure}

\section{Introduction}\label{sec:intro}
At the heart of most machine learning architectures with multi-agent components lies a game played by optimization modules cooperating and competing to minimize an individual loss function.
These settings are ubiquitous; they arise explicitly from the goal of machine learning tasks~(e.g.~chess, poker, Go) or implicitly in the design of architectures such as Generative Adversarial Networks~\citep{goodfellow2014generative,balduzzi2020smooth}.
Game theory has thus emerged as a powerful formalism for studying the broad class of problems under the umbrella of multi-agent machine learning.
In particular, rather than the classical algorithms which simply minimize a single loss function, problems with multiple loss functions require algorithms which
converge to a Nash equilibrium.

Unfortunately, we lack learning algorithms that provably find such equilibria in general games.  
In the presence of multiple interacting loss functions, the standard toolbox of learning algorithms often fails in unpredictable ways.
Recent work has shown that, even under the simplifying assumption of perfect competition (zero-sum games and variants), instead of converging to Nash equilibria the dynamics of standard learning algorithms can cycle~\citep{mertikopoulos2018cycles}, diverge~\citep{BaileyEC18}, or even be formally chaotic~\citep{cheung2019vortices}.
Moreover, when one broadens their scope to a more general class of games, experimental results suggest that chaos is in fact typical behaviour~\citep{sanders2018prevalence} and can even emerge in low-dimensional systems~\citep{palaiopanos2017multiplicative}. 
Considering the ubiquity of multi-agent ML settings alongside these negative results, and others discussed in~\S\ref{subsec:rel_works}, an urgent question arises: \textit{Is there any hope for a general understanding of the behaviours arising from optimization-driven dynamics in games?}

This paper provides evidence that the answer to this question is almost certainly ``no''.
We show that the dynamics of even, arguably, the most well-studied evolutionary learning algorithms, even in a simple and seemingly very constrained class of games, can approximate arbitrarily complex dynamical systems.
\vspace{.2cm}\\
{\bf Informal Main Theorem.}
\textit{Replicator learning dynamics on a matrix game can approximate essentially \emph{any} $C^1$ dynamical system with arbitrary precision.}
\vspace{.2cm}

The significance of our result is clear when one considers that matrix games are a very restricted class of games and that replicator dynamics is a special case of Follow-the-Regularized-Leader (FTRL) dynamics, which captures multiple popular learning algorithms such as gradient descent, hedge, multiplicative weights, etc.~as a special case and enjoys vanishing regret at a rate of $O(1/T)$ (see e.g.~\cite{mertikopoulos2018cycles} and the citations within).
Since matrix games are a simple class of games and replicator dynamics is a special case of FTRL, then the dynamics of more general games and learning algorithms cannot be any simpler than the case we have shown to be arbitrarily complex; framed like this, our result can be interpreted in the same fashion as a reduction in computational complexity theory.
When understood in this way, our result implies that understanding learning dynamics in multi-agent machine learning settings is akin to a general understanding of dynamical systems and has multi-faceted implications depending on the context it is being interpreted in.
We discuss a range of such implications in~\S\ref{sec:disc}, including what our result implies for designing learning algorithms and the use of regret for measuring learning performance.

A formal statement of the main theorem requires carefully exploring strong notions of equivalence between dynamical systems, and how these notions can be used to meaningfully define approximations of dynamical systems.
All of the requisite language and formalism is introduced in~\S\ref{sec:prelims}, while our notion of approximation and the main result is given in~\S\ref{sec:universal_learning_dynamics}.
Our proof establishes connections between game theory, topological dynamics, learning theory, and standard population models from mathematical ecology. 
In fact, the question explored in this paper mirrors one famously asked nearly fifty years ago by~\cite{smale1976differential}---about whether it is possible to meaningfully understand and predict the behaviours of well-studied ecological models of competition. 
The class of systems considered in this paper are, relatively speaking, even more restricted than the ones considered by Smale in his construction.
However, as we show through a sequence of transformations and embeddings, we can approximately capture the behaviour of any target system using replicator dynamics in finite dimensional matrix games.

\subsection{Related Work}\label{subsec:rel_works}
Optimization-driven learning in games, e.g., regret-minimizing dynamics, has been the subject of intense study.
The standard approach focuses on their time-averaged behaviour and its convergence to coarse correlated equilibria in games, (see e.g.~\cite{roughgarden2015intrinsic,stoltz2007learning}).
The analysis of the time-averaged behaviour, however, is unable to faithfully capture the day-to-day dynamics.
In many cases, it has been shown that the emergent day-to-day behaviour is non-convergent in a strong formal sense~\citep{mertikopoulos2018cycles,bailey2019multi}.
Perhaps even more alarming is the fact that strong time-average convergence guarantees may hold true regardless of whether the underlying system is convergent, recurrent, or even chaotic~\citep{palaiopanos2017multiplicative,CFMP2019,Thip18,cheung2019vortices,cheung2020chaos,bailey2020finite}. In fact, all FTRL dynamics, despite their optimal regret guarantees, fail to achieve (even local) asymptotic stability on \textit{any} (even partially) mixed Nash equilibrium in effectively \textit{all} games \citep{flokas2020no}.

With the proliferation of multi-agent architectures in machine learning, e.g., Generative Adversarial Networks (GANs), recent work has placed particular attention on the modes of failure arising in variants of zero-sum competition between learning agents (e.g.~between two neural networks).
In zero-sum games the dynamics of standard learning algorithms such as gradient descent do not converge to Nash equilibria.
Instead, the resultant dynamics may lead to cycling~\citep{mertikopoulos2018cycles,vlatakis2019poincare,boone2019darwin,balduzzi2018mechanics}, divergence~\citep{BaileyEC18,cheung2018multiplicative}, or formally chaotic behaviours~\citep{cheung2019vortices,cheung2020chaos}.
In the face of such strong negative results for out-of-the-box optimization methods the development of tailored algorithmic solutions is incentivized, e.g.~\cite{daskalakis2018training,mertikopoulos2019optimistic,gidel2019a,mescheder2018training,perolat2020poincar,yazici2018unusual}.
However, even when these algorithms do equilibrate, they may stabilize at fixed points that are not Nash equilibria and thus not game theoretically meaningful~\citep{adolphs2018local,daskalakis2018limit}. 

Alongside studies of learning in zero-sum games, differential games (i.e.~smooth games) have been the focus of recent research as a powerful, and more general, model of multi-agent machine learning (e.g.~\cite{balduzzi2018mechanics,mazumdar2020gradient}).
\cite{letcher2019differentiable} leveraged connections with Hamiltonian dynamics to design new algorithms for training GANs while ``correcting'' cyclic behaviours.
In addition, \cite{balduzzi2020smooth} explored the structure of differential games and revealed promising training guarantees when relatively weak constraints are placed on the loss functions of agents in the model and the payoff structure of their interactions.
Within the space of differential games, the dynamics of non-convex non-concave games have received particular attention and a number of distinct non-equilibrating failure modes have been catalogued~\citep{vlatakis2019poincare,hsieh2020limits}. 
The impossibility of universal algorithmic solutions within the broad scope of differential games has also been reinforced by recent work that constructs a simple example where reasonable gradient-based methods cannot hope to converge~\citep{letcher2021impossibility}.

Even when one restricts their attention on matrix games, the difficulty of learning Nash equilibria grows significantly and swiftly when one broadens their scope to a more general class of games than just zero-sum games~\citep{daskalakis2010learning,paperics11,galla2013complex,papadimitriou2019game}. 
In fact, detailed experimental studies suggest that chaos is standard fare~\citep{sanders2018prevalence} and emerges even in very low dimensional systems~\citep{sato2002chaos,palaiopanos2017multiplicative, 2017arXiv170109043P}. 
This abundance of non-equilibrating results has inspired a program for linking game theory to topology of dynamical systems \citep{Entropy18,papadimitriou2019game}, specifically to Conley's fundamental theorem of dynamical systems \citep{conley1978isolated}.
This approach shifts attention from Nash equilibria to a more general notion of recurrence, called chain recurrence, that is flexible enough to capture both cycling behavior as well as chaos. These tools have since found application in multi-agent ML settings \citep{omidshafiei2019alpha,rowland2019multiagent}. 

Thus far, almost all work on learning dynamics in games can be roughly broken into two streams: (i) designing algorithms that converge to desirable states, and (ii) characterizing the possible emergent behaviors from a given class of game dynamics.
Our work differs from both these lines of inquiry by, in a sense, doing the \emph{converse} of (ii).
Roughly speaking, we ask the question ``Given a target dynamical system, can we construct a game whose learning dynamics behave in a similar fashion?''
To the best of our knowledge, our work is the first construction of this sort in the context of learning in games.
Our approach is inspired by the work of~\cite{smale1976differential} and~\cite{hirsch1988systems} in mathematical ecology, which have developed constructions in the same spirit as ours to study the  dynamics of population models.

\section{Preliminaries}\label{sec:prelims}
\subsection{Game Theory}\label{subsec:game_theory}
A \emph{matrix game} (\emph{finite $2$-player normal form game}) is defined on a set of two agents $[2] = \{1,2\}$.
Agent $i \in [2]$ chooses actions from a finite action set $S_i$ according to a distribution $\x_i$ in the probability $|S_i|$-simplex $\Delta^{|S_i|} = \{\x_i \in \reals^{|S_i|}_+ : \sum_{s \in S_i} x_{i s} = 1\}$.
The probability distribution $\x_i$ is known as $i$'s \emph{mixed strategy}.
As the name indicates, agents in a matrix game receive payoffs according to a payoff matrix $A_{i,j} \in \reals^{|S_i| \times |S_j|}$ where $i,j \in [2]$ and $i \neq j$.
Given that mixed strategies $\x_1 \in \Delta^{|S_1|}$ and $\x_2 \in \Delta^{|S_2|}$ are chosen, agent $1$ receives payoff $\x^\Trans_1 A_{1,2} \x_2$ and agent $2$ receives payoff $\x^\Trans_2 A_{2,1} \x_1$.
This gives rise to two optimization problems, one per agent, where agents act strategically and independantly to maximize their expected payoff over the other agent's mixed strategy, i.e.
\begin{equation}\label{eq:game_opt}
    \max\limits_{\x_i \in \Delta^{|S_i|}} \x_i^\Trans A_{i,j} \x_j, \qquad i,j \in [2]~\&~i \neq j~.
\end{equation}

\subsection{Follow-the-Regularized-Leader (FTRL) Learning and Replicator Dynamics}\label{subsec:learning}

Arguably the most well known class of algorithms for online learning and optimization is Follow-the-Regularized-Leader (FTRL). 
Given initial payoff vector $\y_i(0)$, an agent $i$ that plays against agent $j$ in a matrix game $A_{i,j}$
updates their strategy at time $t$ according to
\begin{equation}
\label{eqn:FTRL}
\begin{aligned}
\y_i(t)&= \y_{i}(0)+\int_0^t A_{i,j}\x_{j}(s) ds \\
\x_i(t)&= \argmax_{\x_i\in \Delta^{|S_i|}} \{\langle \x_i, \y_i(t)\rangle-h_i(\x_i)\} 
\end{aligned}
\end{equation}
where $h_i$ is strongly convex and continuously differentiable. FTRL effectively performs a balancing act between exploration and exploitation.
The accumulated payoff vector $\y_i(t)$ indicates the total payouts until time $t$, i.e.~if agent $i$ had played strategy $s_i \in S_i$ continuously from $t=0$ until time $t$, agent $i$ would receive a total reward of $\y_{is_i}(t)$. 
The two most well-known instantiations of FTRL dynamics are the online gradient descent algorithm when $h_i(\x_i)=||\x_i||_2^2$,
and the replicator dynamics (the continuous-time analogue of Multiplicative Weights Update~\citep{Arora05themultiplicative}) when $h_i(\x_i)=\sum_{s_i\in {S}_i} \x_{is_i}\ln \x_{is_i}$. 
FTRL dynamics in continuous time has bounded regret in arbitrary games~\citep{mertikopoulos2018cycles}. 
For more information on FTRL dynamics and online optimization, see \cite{Shalev2012}.

In this paper, we focus on \emph{replicator dynamics} (RD) as our main game dynamics. 
Aside from its role in optimization, RD is one of the key mathematical models of evolution and biological competition~\citep{Schuster1983533,taylor1978evolution}. 
It is also the prototypical dynamic studied in the field of evolutionary game theory~\citep{Weibull,Sandholm10}. 
In this context, replicator dynamics can be thought of as a normalized form of the population models introduced in~\S\ref{subsec:population_models}, and is studied given just a single payoff matrix $A$ and a single probability distribution $\x$ that can be thought abstractly as capturing the proportions of different species/strategies in the current population. 
Species/strategies get randomly paired up and the resulting payoff determines which strategies will increase/decrease over time.

Formally, the dynamics are as follows.
Let $A \in \reals^{m \times m}$ be a matrix game and $\x \in \Delta^m$ be the mixed strategy played.
RD on $A$ are given by:
\begin{equation} \label{eq:replicator}
    \dot{x}_i = \frac{d x_i}{d t} = x_i \left((A\x)_i - \x^\Trans A \x \right), \qquad i \in [n]
\end{equation}
Under the symmetry of $A_{i,j}=A_{j,i}$,
and of initial conditions (i.e. $\x_i=\x_j$),
it is immediate to see that under  
the $\x_i, \x_j$ solutions of  (\ref{eqn:FTRL}) are identical to each  other and to the solution of (\ref{eq:replicator}) with $A=A_{i,j}=A_{j,i}$. 
For our purposes, it will suffice to focus on exactly this setting of matrix games defined by a single payoff matrix $A$ and a single probability distribution $\x$, which is actually the standard setting within evolutionary game theory.

\subsection{Dynamical Systems Theory}\label{subsec:dynamics_theory}
Dynamical systems are mathematical models of time-evolving processes. 
The object undergoing change in a dynamical system is called its \emph{state} and is often denoted by $\x \in \X$, where $\X$ is a topological space called a \emph{state space}. 
We will be focusing on \emph{continuous time} systems with time denoted by $t \in \reals$.
Change between states in a dynamical system is described by a \emph{flow} $\Phi: \X \times \reals \to \X$ satisfying two properties:
\begin{enumerate}[label=(\roman*)]
    \item For each $t \in \reals$, $\Phi(\cdot,t): \X \to \X$ is bijective, continuous, and has a continuous inverse.
    \vspace*{-2pt}
    \item For every $s,t \in \reals$ and $\x \in \X$, $\Phi(\x,s+t) = \Phi(\Phi(\x,t),s)$.
\end{enumerate}
Intuitively, flows serve the purpose of describing the evolution of states in the dynamical system.
Given a time $t \in \reals$, the flow describes the relative movement of every point $\x \in \X$; we will denote this by the map $\Phi^t : \X \to \X$.
Similarly, given a point $\x \in \X$, the flow captures the trajectory of $\x$ as a function of time; in an abuse of notation, we will denote this by $\Phi^t(\x)$ where $t$ is changing.

When $\x$ changes according to a continuous function in $t$ the dynamical system is often given as a system of \emph{ordinary differential equations} (ODEs).
Systems of ODEs describe a \emph{vector field} $V:\X \to T\X$ which assigns to each $\x \in \X$ a vector in the tangent space of $\X$ at $\x$.
This fact is particularly important in this paper for the case that $\X$ is $\Delta^n$, in which case the tangent space $T\Delta^n$ at each $\x \in \Delta^n$ is: $\{\y \in \reals^n : \|\y\|_1 = 0\}$ for $\x$ in the interior of $\Delta^n$, and additionally ``pointing inwards'' for $\x$ on the boundary of $\Delta^n$ (i.e.~$y_i \geq 0$ if $x_i = 0$). 
A system of ODEs is said to \emph{generate} (resp.~\emph{give}) a flow $\Phi$ if $\Phi$ describes a solution of the ODEs at each point $\x \in \X$.
Throughout this paper we will assume that all dynamical systems discussed can be given by a system of ODEs.
As such, we will use the term \emph{dynamical system} to refer to the system of ODEs, the associated vector field, and a generated flow interchangeably.
Note that, for Lipschitz-continuous systems of ODEs, the generated flow is unique (see~\cite{perko1991differentialeqsbook,meiss2007differentialbook}) and using these terms interchangeably is well defined.

An important notion in this paper, and dynamical systems theory in general, is that of a \emph{global attracting set} of the dynamical system.
Let $\Phi$ be a flow generated by some dynamical system on $\X$.
We say $\Y \subset \X$ is \emph{forward invariant} for the flow $\Phi$ if $\Phi^t(\y) \in \Y$ for every $t \geq 0$, $\y \in \Y$.
We say $\Y \subset \X$ is \emph{globally attracting} for the flow $\Phi$ if $\Y$ is nonempty, forward invariant, and
\begin{equation}
   \Y \supseteq \bigcap\limits_{t > 0} \{\Phi^t(\x) : \x \in \X\}~.
\end{equation}
Intuitively speaking, if $\Y$ is globally attracting it will capture the dynamics of $\Phi$ starting from any point in $\X$ after some transitionary period of time.
In~\S\ref{sec:universal_learning_dynamics} we also use the notion of \emph{stationary} dynamics, which is often considered ``uninteresting'' in dynamical systems theory since, in a sense, it describes dynamical systems that are not dynamic.
For our purposes, we say a dynamical system is stationary if the ODEs of that system are identically zero, i.e.~the ODEs describe a system whose solutions are stuck in their initial state. 

Now let $\X$ and $\Y$ be two topological spaces.
We say that a function $f: \X \to \Y$ is a \emph{homeomorphism} if (i) $f$ is bijective, (ii) $f$ is continuous, and (iii) $f$ has a continuous inverse.
Furthermore, two flows $\Phi: \X \times \reals \to \X$ and $\Psi: \Y \times \reals \to \Y$ are \emph{homeomorphic} if there exists a homeomorphism $g: \X \to \Y$ such that for each $\x \in \X$ and $t \in \reals$ we have $g(\Phi(\x,t)) = \Psi(g(\x),t)$.
If additionally $g$ is $C^1$ and has a $C^1$ inverse, then we say $g$ is a \emph{diffeomorphism} and that the flows $\Phi$ and $\Psi$ are \emph{diffeomorphic}.
Note that every diffeomorphism is also a homeomorphism, and thus every pair of diffeomorphic flows are also homeomorphic.
Homeomorphisms (resp.~diffeomorphisms) are a strong, and typical, notion of equivalence between dynamical systems.
In essence, two dynamical systems are homeomorphic if their trajectories can be mapped to one another by smoothly stretching and folding space.

\subsection{Ecological Population Models}\label{subsec:population_models}
Throughout this paper we make use of tools developed in mathematical ecology for studying the growth and decline of populations of species.
As is typically done in ecological models, consider vectors $\x \in \reals^n_+$ where $n$ is the number of ``species'' and $x_i$ represents the population of the $i^{th}$ space.
Suppose that the dynamics of each population is given by the system of ODEs
\begin{equation} \label{eq:competition}
    \dot{x}_i = \frac{d x_i}{d t} = x_i M_i(\x), \qquad i \in [n]~.
\end{equation}
We call any dynamical systems given by eq.~\ref{eq:competition} a \emph{population system}.
Furthermore, for each $i \in [n]$, $M_i$ is called the $i^{th}$ species' \emph{fitness function}.

Two well studied special cases of population systems will be particularly relevant to our analysis: (i) when the fitness functions are affine and (ii) when the fitness functions are multivariate generalized polynomials.
In case (i)---when the fitness function $M_i$ is affine for every $i \in [n]$---the system of ODEs is known as the \emph{Lotka-Volterra} (LV) equations and is given by
\begin{equation} \label{eq:LV}
    \dot{x}_i = \frac{d x_i}{d t} = x_i \left(\hat{\lambda}_i +  \sum_{j \in [n]} \hat{A}_{ij} x_j \right), \qquad i \in [n]
\end{equation}
where $\hat{\lambda} \in \reals^n$ and $\hat{A} \in \reals^{n \times n}$.
In case (ii)---when the fitness function $M_i$ is a multivariate \emph{generalized polynomial} for every $i \in [n]$---the system of ODEs is known as the \emph{generalized Lotka-Volterra} (GLV) equations and is given by
\begin{equation} \label{eq:general_LV}
    \dot{x}_i = \frac{d x_i}{d t} = x_i \left(\lambda_i +  \sum_{j \in [m]} A_{ij} \prod_{k \in [n]} x^{B_{jk}}_k \right), \qquad i \in [n]
\end{equation}
where $m$ is some positive integer, $\lambda \in \reals^n$, $A \in \reals^{n \times m}$, and $B \in \reals^{m \times n}$.

\section{Main Result: Universality of Replicator Dynamics in Matrix Games}\label{sec:universal_learning_dynamics}
In this section we formally state and prove our main result.
Specifically, we show that replicator dynamics in finite matrix games can emulate the behaviour of \emph{any} finite dimensional $C^1$ dynamical system defined on a space diffeomorphic to the probability simplex.
In order to state our result, we introduce a notion of approximately embedding one dynamical system into another.
\begin{definition}\label{def:approx_embedding}
A flow $\Psi$ on topological space $\X$ is \emph{$(\epsilon,T)$-approximately embedded} in a flow $\Theta$ on topological space $\Z$ if there exists $\Z' \subseteq \Z$ and topological space $\Y$ satisfying the following: 
\begin{enumerate}[label=(\roman*)]
    \item The diameter of $\Y$ is $1$ with respect to $\|\cdot\|_{\infty}$, i.e.~$\sup\{\|\y - \hat{\y}\|_{\infty} : \y,\hat{\y} \in \Y\} = 1$.
    \item There exists diffeomorphisms $g:\X \to \Y$ and $f:\relint{\Y} \to \Z'$.
    \item For every $\y \in \relint{\Y}$ and $t \in [0,T]$ we have 
    \begin{equation*}
        \|g(\Psi^t(g^{-1}(\y))) - f^{-1}(\Theta^t(f(\y)))\|_{\infty} < \epsilon~.
    \end{equation*}
\end{enumerate}
\end{definition}
This definition can be seen as an extension of embeddings traditionally studied in differential topology; in fact, the function $f$ is an embedding of $\relint{\Y}$ into $\Z$ in the traditional sense.
Intuitively, a flow $\Psi$ is said to be $(\epsilon,T)$-approximately embedded in a flow $\Theta$ if, on some subspace $\Z'$, $\Theta$ stays within $\epsilon$ from a diffeomorphic copy of $\Psi$ for at least $T$ time.
Definition~\ref{def:approx_embedding} stipulates that the approximation is for every $\y \in \relint{\Y}$, instead of on the entire space $\Y$.
The importance of this distinction lies in the fact that the flow $\Theta$ restricted to $\Z'$ should be diffeomorphic to a flow that is well defined in $\Y$, but the boundary of $\Y$ may not be well defined in the embedding space.

With this definition, we can state our main result.
For expository purposes we state Theorem~\ref{thm:main} in terms of the convex hull of $n+1$ affinely independent points in $\reals^n$, since this captures most settings of interest to machine learning practitioners.
Theorem~\ref{thm:main} trivially extends to \emph{any} space diffeomorphic to the simplex since diffeomorphisms are closed under composition.

\begin{figure}[t]
\centering

\definecolor{c1}{RGB}{255, 204, 153}
\definecolor{c2}{RGB}{170, 255, 255}
\definecolor{c3}{RGB}{255, 153, 153}

\tikzstyle{flow} = [rectangle,draw=black,fill=white,thick,inner sep=0pt,minimum size=.7cm,font=\bfseries]
\tikzstyle{space} = [rectangle,draw=black,fill=white,thick,inner sep=0pt,minimum height=9mm, minimum width=1.8cm,font=\small]

\tikzstyle{arrow} = [-,line width=.5mm]
\tikzstyle{sparrow} = [<-,line width=.25mm]
\tikzstyle{relarrow} = [-,dashed,line width=.25mm]
\tikzstyle{textarrow} = [-,line width=.25mm]

\tikzstyle{thm1} = [fill,opacity=.5,fill opacity=.5,line cap=round, line join=round, line width=45pt]
\tikzstyle{thm2} = [fill,opacity=.5,fill opacity=.5,line cap=round, line join=round, line width=30pt]

\makebox[0pt]{
\begin{tikzpicture}
\node[flow] (prescribed) {$\Psi$};
\node[flow] (diffeo) [right = 2.2cm of prescribed] {$\Phi$}
    edge [-,line width=.5mm] node[auto,swap] {{\scriptsize \text{Diffeomorphism}}} (prescribed);
\node[flow] (approx) [right = 2.2cm of diffeo] {$\hat{\Phi}$}
    edge [-,line width=.5mm] node[auto,swap] {{\scriptsize $\epsilon \text{-Approximation}$}} (diffeo);
\node[flow] (glv) [right = 2.2cm of approx] {$\Gamma$}
    edge [-,line width=.5mm] node[auto,swap] {{\scriptsize \text{Embed}}} (approx);
\node[flow] (padglv) [right = 2.2cm of glv] {$\tilde{\Gamma}$}
    edge [-,line width=.5mm] node[auto,swap] {{\scriptsize \text{Embed}}} (glv);
\node[flow] (lv) [right = 2.2cm of padglv]  {$\hat{\Gamma}$}
    edge [-,line width=.5mm] node[auto,swap] {{\scriptsize \text{Diffeomorphism}}} (padglv);
\node[flow] (replicator) [right = 2.2cm of lv]  {$\Theta$}
    edge [-,line width=.5mm] node[auto,swap] {{\scriptsize \text{Diffeomorphism}}} (lv);

\node[space] (original) [below = .85cm of prescribed] {$\X$}
    edge [relarrow] (prescribed);
\node[space] (simplex) [below = .85cm of diffeo] {$\Y = \Delta^n$}
    edge [relarrow] (diffeo)
    edge [<-,line width=.25mm,bend left,out=-5,in=185] node[auto,swap] {{\scriptsize $g$}} (original)
    edge [->,line width=.25mm,bend left,out=5,in=175] node[below left=0.01cm and -.5cm of simplex] {{\scriptsize $g^{-1}$}} (original);
\node[space] (intsimplex) [below = .85cm of approx] {$\relint{\Delta^n}$}
    edge [relarrow] (approx)
    edge [<-,line width=.25mm] node[auto,swap] {{\scriptsize \text{Restrict}}} (simplex);
\node[space] (glvreals) [below = .85cm of glv] {$\reals^n_{++}$}
    edge [relarrow] (glv)
    edge [<-,line width=.25mm] node[auto,swap] {{\scriptsize $\text{id}$}} (intsimplex);
\node[space] (padglvreals) [below = .85cm of padglv] {$\reals^{m-1}_{++}$}
    edge [relarrow] (padglv)
    edge [<-,line width=.25mm] node[auto,swap] {{\scriptsize $f_1$}} (glvreals);
\node[space] (lvreals) [below = .85cm of lv] {$\reals^{m-1}_{++}$}
    edge [relarrow] (lv)
    edge [<-,line width=.25mm] node[auto,swap] {{\scriptsize $f_2$}} (padglvreals);
\node[rectangle,draw=black,fill=white,thick,inner sep=0pt,minimum height=9mm, minimum width=2.6cm,font=\small] (repsimplex) [below = .85cm of replicator] {$\Z = \relint{\Delta^m}$}
    edge [relarrow] (replicator)
    edge [<-,line width=.25mm] node[auto,swap] {{\scriptsize $f_3$}} (lvreals)
    edge [<-,line width=.25mm,bend right,out=-17,in=200] node[above left=0.01cm and -1.25cm of glvreals,swap] {{\tiny $f = f_3 \circ f_2 \circ f_1$}} (glvreals);
    
\node[above=.08cm of prescribed,text width=1.8cm,text centered, font=\scriptsize] {Prescribed Dynamical System}
    edge [textarrow] (prescribed);
\node[above=.08cm of glv,text width=1.8cm,text centered, font=\scriptsize] {GLV System}
    edge [textarrow] (glv);
\node[above=.08cm of padglv,text width=1.8cm,text centered, font=\scriptsize] {GLV System}
    edge [textarrow] (padglv);
\node[above=.08cm of lv,text width=1.8cm,text centered, font=\scriptsize] {LV System}
    edge [textarrow] (lv);
\node[above=.08cm of replicator,text width=1.8cm,text centered, font=\scriptsize] {Replicator Dynamics}
    edge [textarrow] (replicator);

\node[above=2.05cm of glv,text centered, font=\Large] {\textbf{Theorem~\ref{thm:main}}};
\node[above=1.08cm of approx,text centered, font=\large] {\textbf{Theorem~\ref{thm:GLV_approx_of_general_sys}}};
\node[above right=1.08cm and -2.9cm of lv,text centered, font=\large] {\textbf{Theorem~\ref{thm:GLV_to_Replicator}}};
\node[below left = -.18cm and -.48cm of original] (t1bl) {};
\node[below right = -.18cm and -.48cm of repsimplex] (t1br) {};
\node[above = 4.15cm of t1bl] (t1tl) {};
\node[above = 4.15cm of t1br] (t1tr) {};
\node[below left = -.18cm and -.4cm of simplex] (t2bl) {};
\node[below right = -.18cm and -.5cm of glvreals] (t2br) {};
\node[above = 3.5cm of t2bl] (t2tl) {};
\node[above = 3.5cm of t2br] (t2tr) {};
\node[below left = -.18cm and -.55cm of glvreals] (t3bl) {};
\node[below right = -.18cm and -.5cm of repsimplex] (t3br) {};
\node[above = 3.5cm of t3bl] (t3tl) {};
\node[above = 3.5cm of t3br] (t3tr) {};

\begin{pgfonlayer}{background}
\begin{scope}[transparency group,opacity=.5]
\draw[thm1,opacity=1,color=c1] (t1bl) -- (t1br) -- (t1tr) -- (t1tl) -- (t1bl);
\fill[thm1,opacity=1,color=c1] (t1bl.center) -- (t1br.center) -- (t1tr.center) -- (t1tl.center) -- (t1bl.center);
\end{scope}

\begin{scope}[transparency group,opacity=.5]
\draw[thm2,opacity=1,color=c2] (t2bl) -- (t2br) -- (t2tr) -- (t2tl) -- (t2bl);
\fill[thm2,opacity=1,color=c2] (t2bl.center) -- (t2br.center) -- (t2tr.center) -- (t2tl.center) -- (t2bl.center);
\end{scope}

\begin{scope}[transparency group,opacity=.5]
\draw[thm2,opacity=1,color=c3] (t3bl) -- (t3br) -- (t3tr) -- (t3tl) -- (t3bl);
\fill[thm2,opacity=1,color=c3] (t3bl.center) -- (t3br.center) -- (t3tr.center) -- (t3tl.center) -- (t3bl.center);
\end{scope}
\end{pgfonlayer}

\end{tikzpicture}
}

\caption{A diagram highlighting the relationship between Theorems~\ref{thm:main},~\ref{thm:GLV_approx_of_general_sys},~and~\ref{thm:GLV_to_Replicator}, along with the steps used to construct the $(\epsilon,T)$-approximate embedding of $\Psi$ in $\Theta$.
All embeddings are injective smooth maps and ensure the original dynamical system's approximation can be recovered from the higher dimensional embedding spaces.
The functions $f_1, f_2,$ and $f_3$ are diffeomorphisms defined in Appendix~\ref{proof:glv_to_replicator}.
Furthermore, in Theorem~\ref{thm:main}, the subspace $\Z' \subseteq \Z$ for the $(\epsilon,T)$-approximate embedding is $f \circ \text{id} (\relint{\Delta^n})$.}
\label{fig:embedding_diagram}
\end{figure}
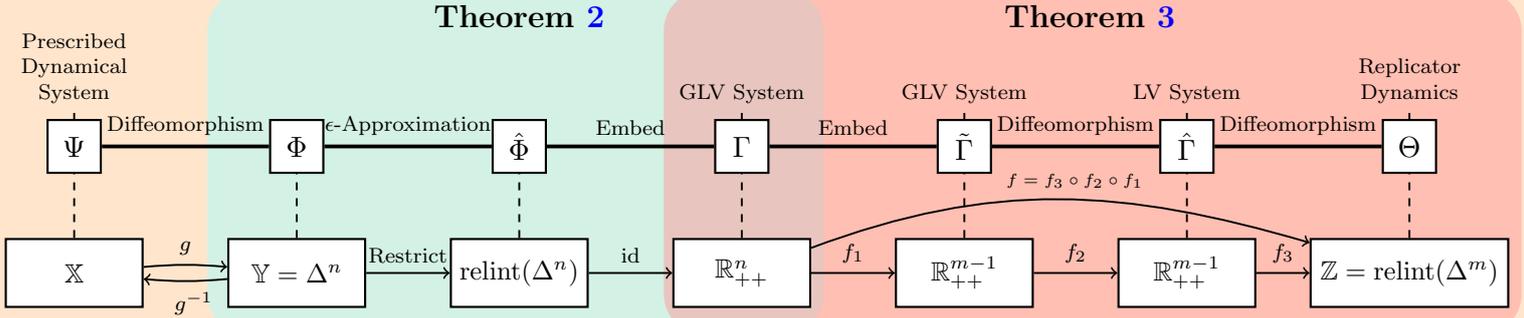

\begin{theorem}\label{thm:main}
Let $\X$ be the convex hull of a set of $n+1$ affinely independent points in $\reals^n$ and $\Psi$ be any flow on $\X$ given by a finite dimensional $C^1$ system of ODEs.
For any $\epsilon,T>0$, there exists $m \geq 0$ and a matrix $A \in \reals^{m \times m}$ such that $\Psi$ is $(\epsilon,T)$-approximately embedded in the flow given by replicator dynamics on $A$.
\end{theorem}

A proof of Theorem~\ref{thm:main} follows immediately from Theorems~\ref{thm:GLV_approx_of_general_sys}~and~\ref{thm:GLV_to_Replicator} stated below.
The basic intuition of how Theorems~\ref{thm:main},~\ref{thm:GLV_approx_of_general_sys},~and~\ref{thm:GLV_to_Replicator} are proved and relate to one another is summarized in Figure~\ref{fig:embedding_diagram}.
The remainder of this section is dedicated to formally proving Theorem~\ref{thm:main}.
We begin by stating Theorems~\ref{thm:GLV_approx_of_general_sys}~and~\ref{thm:GLV_to_Replicator} along with their proof sketches---the full proofs are given in Appendices~\ref{proof:glv_approx} and~\ref{proof:glv_to_replicator} respectively.
We then conclude by demonstrating how these Theorems come together to prove Theorem~\ref{thm:main}.
It is worth noting that in many cases our proof techniques are constructive, and can be used to derive a matrix game that emulates the behaviour of a prescribed dynamical system under RD; a concrete example is given in \S\ref{sec:lorenz} where a matrix game giving rise to the iconic Lorenz system~\citep{lorenz1963deterministic} is constructed.

\begin{theorem} \label{thm:GLV_approx_of_general_sys}
Let $\Phi$ be a flow on $\Delta^n$ that is generated by a $C^1$ system of ODEs.
For any $\epsilon,T > 0$, there exists a flow $\Gamma$ on $\reals^n_{++}$ given by a system of GLV equations (eq.~\ref{eq:general_LV}) such that:
\begin{enumerate}[label=(\roman*)]
    \item A subspace of $\relint{\Delta^n}$ is a global attracting set of $\Gamma$.
    \item For every $\y \in \relint{\Delta^n}$ and $t \in [0,T]$ we have 
    \begin{equation*}
        \|\Phi^t(\y) - \Gamma^t(\y)\|_{\infty} < \epsilon~.
    \end{equation*}
\end{enumerate}
\end{theorem}

Formally constructing a flow $\Gamma$ with the properties stated in Theorem~\ref{thm:GLV_approx_of_general_sys} requires some technical legwork, but the intuition behind our construction of $\Gamma$ is rather straightforward.
First we get a polynomial approximation of the ODEs generating $\Phi$ from the well known Stone-Weierstrass theorem, which we call $\p = (p_1,\dots,p_n)$.
In our construction we ensure that $\p$ generates a forward invariant flow $\hat{\Phi}$ on $\Delta^n$ and has a subspace of $\relint{\Delta^n}$ as a global attracting set.
Then, for each $i \in [n]$, we divide $p_i$ by $y_i$ and add the resultant generalized polynomials to $\pi(\y) = (1 - \|\y\|_1)$, which yields a new generalized polynomial $\pi + \frac{1}{y_i} p_i$ for each $i \in [n]$.
By setting these new generalized polynomials, $\pi + \frac{1}{y_i} p_i$, as the fitness functions of a population system on $\reals^n_{++}$ we get the system generating $\Gamma$. 
The role of $\pi$ is to define logistic equation dynamics between the ODEs so that the dynamics of the system as a whole approaches $\Delta^n$.
Since the logistic equation ensures $\|\y\|_1 \to 1$ as $t \to \infty$, though the dynamics outside $\Delta^n$ may be different from those on $\Delta^n$, this construction ensures that the probability simplex $\Delta^n$ is attracting all of the dynamics.
Furthermore, not only is $\Delta^n$ forward invariant under the construction, but $\pi(\x) = 0$ for $\x \in \Delta^n$ and so the flow is exactly generated by the polynomials $\p$ that approximate $\Phi$.
A full proof of Theorem~\ref{thm:GLV_approx_of_general_sys} can be found in Appendix~\ref{proof:glv_approx}.

\begin{theorem} \label{thm:GLV_to_Replicator}
Let $\bar{\lambda} \in \reals^n$, $\bar{A} \in \reals^{n \times (m-1)}$, and $\bar{B} \in \reals^{(m-1) \times n}$ define a system of GLV equations~(eq.~\ref{eq:general_LV}) on $\reals^n_{++}$, where $m-1 \geq n$.
Let $\Gamma$ on $\reals^n_{++}$ be the flow generated by this system of GLV equations.
There exists a flow $\Theta$ on $\relint{\Delta^m}$ and a diffeomorphism $f:\reals^n_{++} \to \P \subseteq \relint{\Delta^m}$ such that:
\begin{enumerate}[label=(\roman*)]
    \item The flow $\Theta$ on $\relint{\Delta^m}$ is given by RD on a matrix game with payoff matrix $A \in \reals^{m \times m}$.
    \item The flow $\Theta|_{\P} = f(\Gamma)$ and $\Gamma = f^{-1}(\Theta|_{\P})$, where $\Theta|_{\P}$ is the flow given by $\Theta$ restricted to $\P$.
\end{enumerate}
\end{theorem}

A full proof of Theorem~\ref{thm:GLV_to_Replicator} appears in Appendix~\ref{proof:glv_to_replicator}.
The result follows from our construction of the payoff matrix $A~\in~\reals^{m \times m}$, which requires an intermediary step where the system of GLV equations on $\reals^{n}_{++}$ is embedded into a system of LV equations on $\reals^{m-1}_{++}$.
This embedding is guaranteed to exist due to a trick introduced by~\citet{brenig1989universal}.
First, the embedding trick adds dummy dimensions to the GLV system by padding $\bar{\lambda}$, $\bar{A}$, and $\bar{B}$ to define a qualitatively equivalent system of GLV equations on $\reals^{m-1}_{++}$---this step ensures the new GLV system is always stationary on the $m-n-1$ newly introduced dimensions and is identical to the original system on a submanifold of $\reals^{m-1}_{++}$.
Next, the embedding trick uses a diffeomorphism to transform the enlarged GLV equations on $\reals^{m-1}_{++}$ into a system of LV equations on $\reals^{m-1}_{++}$.
As summarized in Figure~\ref{fig:embedding_diagram}, the original GLV equations on $\reals^{n}_{++}$ generate a flow $\Gamma$ and we use the embedding trick to place it into a flow $\hat{\Gamma}$ given by the LV equations.
Using a diffeomorphism by~\cite{hofbauer1998book}, which maps trajectories of LV equations in $\reals^{m-1}_{++}$ to trajectories of RD in $\relint{\Delta^m}$, we construct the game matrix $A \in \reals^{m \times m}$.
This matrix $A$ under RD generates the flow $\Theta$ on $\relint{\Delta^m}$ that we are ultimately interested in.
Finally, to find the subspace $\P \subseteq \relint{\Delta^m}$ and diffeomorphism $f$, we note that the embedding constructed by the embedding trick is an injective smooth map from $\reals^{n}_{++}$ to $\reals^{m-1}_{++}$; define $\hat{f}:\reals^{n}_{++} \to \relint{\Delta^m}$ as the composition of this embedding map with the diffeomorphism by~\cite{hofbauer1998book}.
The subspace $\P$ is precisely $\hat{f}(\reals^{n}_{++})$ and the diffeomorphism $f$ is obtained by restricting the range of $\hat{f}$ to $\P$.

With Theorems~\ref{thm:GLV_approx_of_general_sys}~and~\ref{thm:GLV_to_Replicator} stated we are now ready to prove Theorem~\ref{thm:main}.
Let $\X$ be a topological space with a diffeomorphism $g : \X \to \Delta^n$ and let $\Psi$ be \emph{any} flow on $\X$ given by a finite dimensional $C^1$ system of ODEs.
Define the flow $\Phi = g(\Psi)$ on $\Delta^n$, i.e.~the dynamical system diffeomorphic to $\Psi$ via $g$.
From Theorem~\ref{thm:GLV_approx_of_general_sys} we know that for any $\epsilon,T > 0$ there exists a flow $\Gamma$ given by a system of GLV equations on $\reals^n_{++}$ such that $\|\Phi^t(\y) - \Gamma^t(\y)\|_{\infty} < \epsilon$ for every $\y \in \relint{\Delta^n}$ and $t \in [0,T]$.
From Theorem~\ref{thm:GLV_to_Replicator}, for $m \geq n$, we know there exists a flow $\Theta$ on $\relint{\Delta^m}$ and diffeomorphism $f:\reals^n_{++}~\to~\P~\subseteq~\relint{\Delta^m}$ such that $\Theta$ restricted to $\P$ is diffeomorphic to $\Gamma$ via $f$.
Let $\Theta|_{\P}$ be the flow given by $\Theta$ restricted to $\P$.
Since $g(\Psi) = \Phi$, $f^{-1}(\Theta|_{\P}) = \Gamma$, and $f(\relint{\Delta^n}) \subset \P = f(\reals^n_{++})$, it follows that $\|g(\Psi^t(g^{-1}(\y))) - f^{-1}(\Theta^t(f(\y)))\|_{\infty} < \epsilon$ for every $\y \in \relint{\Delta^n}$ and $t \in [0,T]$.
Thus, by setting $\Z = \relint{\Delta^m}$, $\Z' = f(\relint{\Delta^n})$, and $\Y = \Delta^n$, we have shown that $\Psi$ is $(\epsilon,T)$-approximately embedded in $\Theta$.
Furthermore, from Theorem~\ref{thm:GLV_to_Replicator} we know that $\Theta$ is the flow given by replicator dynamics on a matrix game with payoff matrix $A \in \reals^{m \times m}$.
The convex hull of $n+1$ affinely independent points in $\reals^n$ is a special case of $\X$, so we have proven Theorem~\ref{thm:main}.

\section{The Lorenz Game}\label{sec:lorenz}
To demonstrate how the construction in~\S\ref{sec:universal_learning_dynamics} can be applied, we will highlight the construction of a matrix game whose dynamics under RD embeds the iconic system of~\citet{lorenz1963deterministic}; the full construction of this matrix game can be found in Appendix~\ref{append:Lorenz}. 
The Lorenz system's strange attractor, the ``butterfly'', has nearly become synonymous with chaotic flows and is given by the following three dimensional the system of ODEs in $\reals^3$
\begin{align*}
  \dot{x}_1 &= \sigma(x_2 - x_1) \\
  \dot{x}_3 &= x_1(\rho-x_3)-x_2 \\
  \dot{x}_3 &= x_1 x_2 - \beta x_3~,
\end{align*}
where $\sigma,\rho,\beta > 0$ are constants.
Due to the fame of the Lorenz attractor it has been studied extensively and analyses of its dynamics under various settings of its parameters can be found in many sources~(see e.g.~\cite{hateley2019lorenz}).
We will focus on the setting first studied by Lorenz, when $\rho = 28$, $\sigma = 10$, and $\beta = 8/3$.
Given these parameters, it is straightforward to show that, for sufficiently large $r>0$, the sphere $\mathcal{R} = \{(x_1,x_2,x_3): x_1^2 + x_2^2 + (x_3 - \rho - \sigma)^2 = r\}$ is globally attracting and forward invariant under the Lorenz system.
(Moreover, all initial conditions converge to $\mathcal R$ exponentially fast.)

Shifting the solutions of the Lorenz equation by $r$ in the positive direction for all three dimensions, and then rearranging terms, we arrive at the following GLV system on $\reals^3_{++}$:

\vskip8pt
\noindent\hspace{5pt}\scalebox{0.9}{
\begin{math}
\begin{aligned}
  \dot{x}_1 &= \sigma \left((x_2-r) - (x_1-r)\right) \\
  \dot{x}_3 &= (x_1-r) \left(\rho - (x_3-r) \right)-(x_2-r) \\
  \dot{x}_3 &= (x_1-r) (x_2-r) - \beta (x_1-r)\\
\end{aligned}
\quad\Longrightarrow\quad
\begin{aligned}
  \dot{x}_1 &= x_1 \left(\sigma x_2 x_1^{-1} - \sigma \right) \\
  \dot{x}_2 &= x_2 \left(\eta x_1 x_2^{-1} - x_1 x_3 x_2^{-1} + r x_3 x_2^{-1} + \alpha x_2^{-1} - 1 \right) \\
  \dot{x}_3 &= x_3 \left(x_1 x_2 x_3^{-1} - r x_1 x_3^{-1} - r x_2 x_3^{-1} + \mu x_3^{-1} - \beta \right)
\end{aligned}
\end{math}
}
\vskip8pt

\noindent
where $\eta = \rho+r$, $\alpha = r - \rho r - r^2$, and $\mu = r^2 + \beta r$.
Since we can rewrite the shifted Lorenz system in this GLV form, there is no need to derive the approximation highlighted in Theorem~\ref{thm:GLV_approx_of_general_sys} and we can immediately apply Theorem~\ref{thm:GLV_to_Replicator}.

From the construction used to prove Theorem~\ref{thm:GLV_to_Replicator}, we get the game matrix $A \in \reals^{11 \times 11}$ that can be written as 
\begin{equation*}
    A = 
    \begin{bmatrix}
    -\sigma & \eta & -1 & r & \alpha & 0 & 0 & 0 & 0 & (\sigma - 1) & 0\\
    \sigma & -\eta & 1 & -r & -\alpha & 0 & 0 & 0 & 0 & (1 - \sigma) & 0\\
    \sigma & -\eta & 1 & -r & -\alpha & 1 & -r & -r & \mu & (1 - \sigma - \beta) & 0\\
    0 & -\eta & 1 & -r & -\alpha & 1 & -r & -r & \mu & (1 - \beta) & 0\\
    0 & -\eta & 1 & -r & -\alpha & 0 & 0 & 0 & 0 & 1 & 0\\
    \sigma & \eta & -1 & r & \alpha & -1 & r & r & -\mu & (\beta - \sigma - 1) & 0\\
    \sigma & 0 & 0 & 0 & 0 & -1 & r & r & -\mu & (\beta - \sigma) & 0\\
    0 & \eta & -1 & r & \alpha & -1 & r & r & -\mu & (\beta - 1) & 0\\
    0 & 0 & 0 & 0 & 0 & -1 & r & r & -\mu & \beta & 0\\
    0 & 0 & 0 & 0 & 0 & 0 & 0 & 0 & 0 & 0 & 0 \\
    0 & 0 & 0 & 0 & 0 & 0 & 0 & 0 & 0 & 0 & 0
    \end{bmatrix}~.
\end{equation*}
The solution of RD on $A$ is plotted in Figure~\ref{fig:lorenz}.
It is worth noting that the last row and column are all zeros since they correspond to a compactifying dimension added during our construction for normalizing each dimension.
Similarly, the second to last row of zeros and its corresponding column serves the role of keeping track of the constants in the shifted Lorenz system.
In addition to $A$, we have a diffeomorphism $f: \reals^3_{++} \to \P \subset \relint{\Delta^{11}}$ from $\x \in \reals^3_{++}$ to $\p \in \relint{\Delta^{11}}$ that is written as
\begin{equation*}
    f(\x) = \left(\frac{x_1^{-1} x_2^{1}}{N},\frac{x_1^{1} x_2^{-1}}{N},\frac{x_1^{1} x_2^{-1} x_3^{1}}{N},\frac{x_2^{-1} x_3^{1}}{N},\frac{x_2^{-1}}{N},\frac{x_1^{1} x_2^{1} x_3^{-1}}{N},\frac{x_1^{1} x_3^{-1}}{N},\frac{x_2^{1} x_3^{-1}}{N},\frac{x_3^{-1}}{N},\frac{1}{N},\frac{1}{N} \right)~,
\end{equation*}
where $N$ is a normalization factor given by the sum of the numerators in $f$.
Since we were able to rewrite the Lorenz system in GLV form exactly, without an approximation step, RD on this game is a true embedding of the Lorenz system's strange attractor.

\section{Discussion}\label{sec:disc}
In this paper we show that learning dynamics in finite matrix games can be as complex as any system of ODEs defined on a set diffeomorphic to the probability simplex.
This result has multiple implications for both multi-agent machine learning and algorithmic game theory.

\subsection{Designing Games, Not Just Algorithms}\label{subsec:designing_games}
The complex dynamics that arise from training multi-loss machine learning models, such as Generative Adversarial Networks (GANs), has recently become the object of intense study.
As highlighted in~\S\ref{subsec:rel_works}, this study has led to several results reporting possible modes of failure and algorithms seeking to correct pathological training behaviours.
Our main result, Theorem~\ref{thm:main}, shows that essentially any dynamics can arise in highly simplistic games and thus it is unreasonable to expect any general algorithmic solution to exist for multi-agent machine learning.
As this realization only becomes grimmer when considering possibilities in more complex games, our result drives home a clear message:
multi-agent machine learning has the capability of modelling essentially any process, but this capability comes at the price of interpretability if designing the underlying game is left as an afterthought. Maybe the games themselves should evolve  over time so as to help  guide multi-agent learning~\citep{leibo2019autocurricula,skoulakis2020evolutionary}.

\subsection{Implications on Hardness of Nash Equilibria}\label{subsec:agt}
Our results offer an interesting conclusion to
a progression of results from algorithmic game theory which have established the hardness of computing Nash equilibria.
First, it was shown that computing Nash equilibria is impractical or impossible in general, as it is a PPAD-hard problem~\citep{daskalakis2006complexity}.
Next, it was shown that learning dynamics do not converge to equilibria in general~\citep{daskalakis2010learning,sato2002chaos,mertikopoulos2018cycles,flokas2020no}.
Recently it was revealed that, not only is convergence to equilibria not guaranteed, learning dynamics in games can even be provably chaotic~\citep{palaiopanos2017multiplicative,CFMP2019,Thip18,cheung2019vortices,cheung2020chaos}.
In this paper, we show that, indeed, learning dynamics can effectively simulate \emph{any} behavior even in the special case of finite matrix games.

\begin{figure}[t]
  \centering
  \includegraphics[scale=.305]{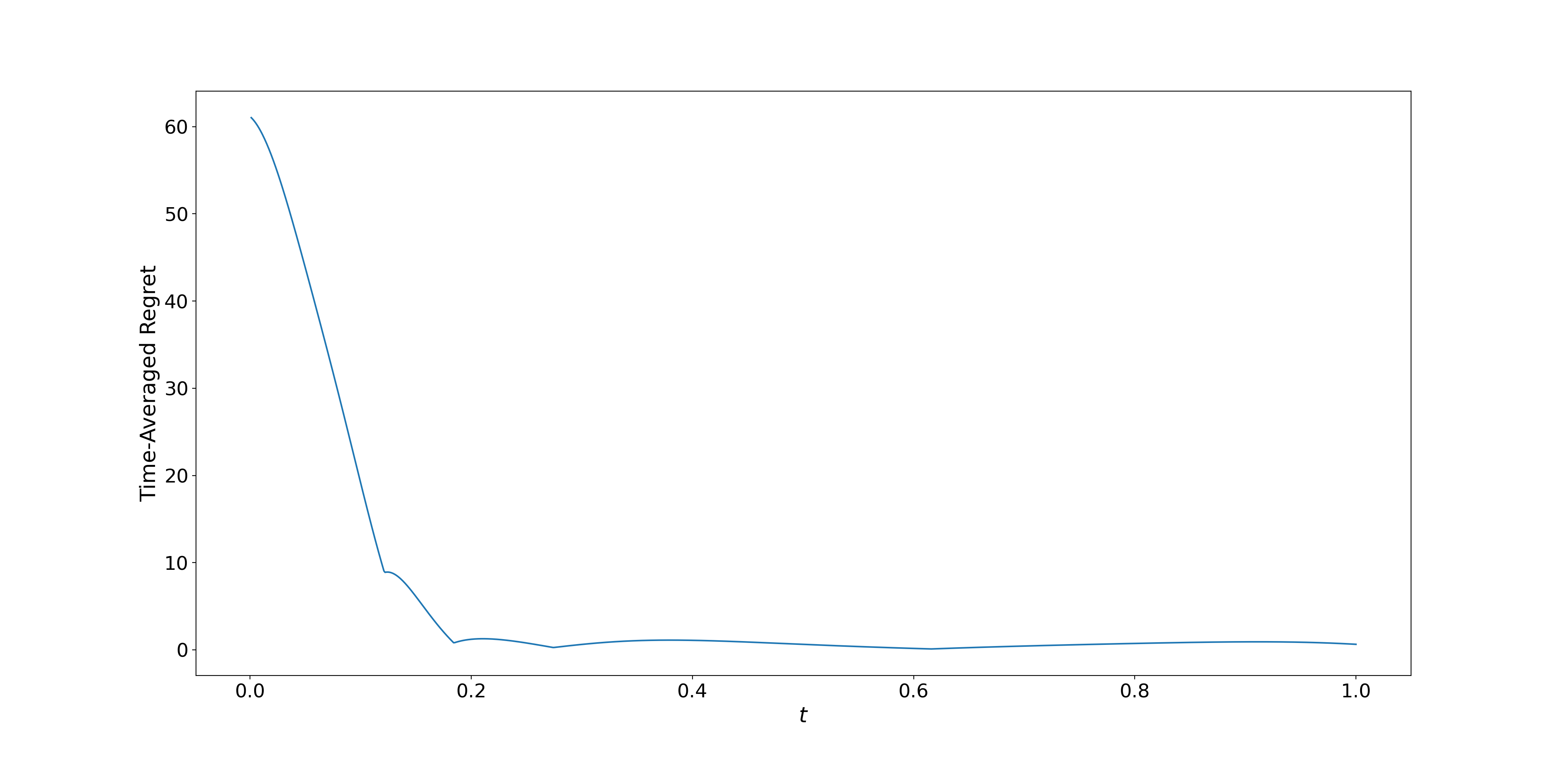}
  \caption{Time-averaged regret of the trajectory shown in Figure~\ref{fig:lorenz}. 
  As theoretical predictions suggest~\citep{mertikopoulos2018cycles}, we see that regret quickly converges to zero.}
  \label{fig:no_regret}
\end{figure}
\subsection{No-Regret Strange Attractors}\label{subsec:regret}
A popular measure of performance for online learning algorithms is \emph{regret}, which measures the difference between an algorithm's \emph{average} performance against the performance of the best \emph{fixed} strategy in hindsight.
When the regret of an algorithm tends to zero as $t \to \infty$ for all sets of input, the algorithm is said to be \emph{no-regret}.
Though analyzing an algorithm's regret provides useful insights and knowing that an algorithm has no-regret is a good guarantee to have, our result shows that having a no-regret algorithm provides effectively no insight into the system's day-to-day behaviour.
The arbitrary behaviour of no-regret learning algorithms in games is perhaps best exemplified by our construction of the Lorenz game in~\S\ref{sec:lorenz}.
Since RD is known to have no-regret in arbitrary games~(e.g.,~\cite{mertikopoulos2018cycles}) and we have embedded the Lorenz system's strange attractor into RD on a matrix game, it follows that we have constructed a game where it is possible to have no-regret while the day-to-day dynamics move along a strange attractor---this is demonstrated in Figure~\ref{fig:no_regret}.
To the best of our knowledge this is the first instance where this possibility has been formally established.
 
\bigskip
We hope that our work inspires further investigations in each of these directions, as we keep exploring the impressive expressive power of multi-agent learning dynamics.

\section*{Acknowledgements}
Georgios Piliouras gratefully acknowledges grant PIE-SGP-AI-2020-01, NRF2019-NRF-ANR095 ALIAS grant and NRF 2018 Fellowship NRF-NRFF2018-07.

\bibliographystyle{plainnat}
\bibliography{ms}

\appendix
\section{Proofs}\label{append:proofs}

\subsection{Proof of Theorem~\ref{thm:GLV_approx_of_general_sys}}\label{proof:glv_approx}
\setcounter{theorem}{1}
\begin{theorem}
Let $\Phi$ be a flow on $\Delta^n$ that is generated by a $C^1$ system of ODEs.
For any $\epsilon,T > 0$, there exists a flow $\Gamma$ on $\reals^n_{++}$ given by a system of GLV equations (eq.~\ref{eq:general_LV}) such that:
\begin{enumerate}[label=(\roman*)]
    \item A subspace of $\relint{\Delta^n}$ is a global attracting set of $\Gamma$.
    \item For every $\y \in \relint{\Delta^n}$ and $t \in [0,T]$ we have 
    \begin{equation*}
        \|\Phi^t(\y) - \Gamma^t(\y)\|_{\infty} < \epsilon~.
    \end{equation*}
\end{enumerate}
\end{theorem}
\begin{proof}
Suppose the flow $\Phi$ is given by a system of ODEs $\h : \Delta^n \to \reals^n$, i.e. $\dot{y}_i = h_i(\y)$.
We will first construct a flow $\hat{\Phi}$ that approximates $\Phi$ within $\epsilon$ on $\Delta^n$ for any $T,\epsilon > 0$.
Importantly, our construction ensures that $\hat{\Phi}$ is given by a system of polynomials that is well defined on $\Delta^n$.
To construct $\hat{\Phi}$ in a way where these properties are satisfied, we find a polynomial approximation of $\Phi$ and then add correction terms to the approximation that ensure the resultant polynomials are well behaved on the boundary of $\Delta^n$.
We conclude the proof by constructing the flow $\Gamma$ in the Theorem statement, which has $\hat{\Phi}$ embedded on a globally attracting set of $\Gamma$.
Note that the construction of the population system giving $\Gamma$ is related to the construction in~\cite{smale1976differential}, where it is shown that some additional bookkeeping guarantees that $\Gamma$ will satisfy properties commonly used in mathematical ecology for modeling species in competition.

The Stone-Weierstrass Theorem famously implies that any continuous function on a compact topological space can be approximated to an arbitrary degree of accuracy with a continuous sequence of polynomials.
It follows that, by the Stone-Weierstrass Theorem, for any $\delta > 0$ and every $i \in [n-1]$ there exists a polynomial $\hat{p}_i$ such that for every $\y \in \Delta^n$ we have $|h_i(\y) - \hat{p}_i(\y)| < \frac{\delta}{n^2}$.
Furthermore, since $\h$ is defined on $\Delta^n$ and therefore has the tangent space $T\Delta^n$ defined in~\S\ref{subsec:dynamics_theory}, we know that $h_n = - \sum_{i \in [n-1]} h_i$, and setting $\hat{p}_n = - \sum_{i \in [n-1]} \hat{p}_i$ guarantees that
\begin{equation*}
    |h_n(\y) - \hat{p}_n(\y)| \leq \sum_{i \in [n-1]}\left|\hat{p}_i(\y) - h_i(\y)\right| < \delta/n~.
\end{equation*}

It is mentioned in~\S\ref{subsec:dynamics_theory} that a dynamical system must ``point inwards'' on the boundary for it to be well defined on $\Delta^n$.
Therefore let us now consider the behaviour of $\hat{p}_i$ on the boundary of $\Delta^n$, i.e.~$\y \in \Delta^n$ such that some $y_i = 0$.
We know that, for each $i \in [n-1]$, $h_i(\y) \geq 0$ for $\y \in \Delta^n$ such that $y_i = 0$, therefore $\hat{p}_i(\y) > -\delta/n^2$.
Similarly, $h_n(\y) \geq 0$ for $\y \in \Delta^n$ such that $y_n = 0$, therefore $\hat{p}_n(\y) > -\delta/n$.
It follows that to use $\hat{p}_i$ to construct a polynomial approximation on $\Delta^n$ of the flow $\Phi$, we will need to add an appropriate correction term to each $\hat{p}_i$.
Define $p_i(\y)~=~\hat{p}_i(\y) + \delta(\tfrac{1}{n} - y_i)$ for $i \in [n-1]$, and $p_n(\y) = -\sum_{i \in [n-1]} p_i(\y)$ as before.
Observe that we have $p_n(\y) = - \sum_{i \in [n-1]} (\hat p_i(\y) + \delta(\tfrac{1}{n} - y_i)) = \hat p_n - \delta \tfrac{n-1}{n} + \delta \sum_{i \in [n-1]} y_i = \hat p_n + \delta(\tfrac{1}{n} - y_n)$.
It follows that, for $i \in [n]$ and $\y \in \Delta^n$, we have
\begin{equation*}
    |h_i(\y) - p_i(\y)| = |h_i(\y) - \hat{p}_i(\y) - \delta (\tfrac{1}{n} - y_i)| < 2\delta~.
\end{equation*}
Furthermore, for $i \in [n]$, when $\y$ is on the boundary of $\Delta^n$ this construction ensures that $p_i(\y) > 0$ when $y_i = 0$ and that $p_i(\y) < 0$ when $y_i = 1$. 
We therefore know the dynamical system given by $\p=(p_1,\dots,p_n)$ is well defined on $\Delta^n$ and has a subspace of $\relint{\Delta^n}$ as a global attracting set.

Let $\hat{\Phi}$ be the flow given by the approximating polynomials $\p=(p_1,\dots,p_n)$.
By definition $\frac{d \Phi^t}{d t}(\y) = \h(\Phi^t(\y))$ and $\frac{d \hat{\Phi}^t}{d t}(\y) = \p(\hat{\Phi}^t(\y))$ for every $\y \in \Delta^n$ and $t \in \reals$.
Furthermore, since $\h$ is $C^1$ and $\Delta^n$ is compact, we know that $\h$ is Lipschitz continuous. 
Letting $L$ denote the Lipschitz constant for $\h$ with respect to $\|\cdot\|_\infty$, it follows that for every $\y \in \Delta^n$ and $t \in \reals$,
\begin{align*}
    \|\Phi^t(\y) - \hat{\Phi}^t(\y)\|_{\infty} &= \int_{0}^{t} \|\frac{d \Phi^s}{d s}(\y) - \frac{d \hat{\Phi}^s}{d s}(\y)\|_{\infty}~ds \\
    &= \int_{0}^{t} \|\h(\Phi^s(\y)) - \p(\hat{\Phi}^s(\y))\|_{\infty}~ds \\
    &\leq \int_{0}^{t} \|\h(\Phi^s(\y)) - \h(\hat{\Phi}^s(\y))\|_{\infty}~ds + \int_{0}^{t} \|\h(\hat{\Phi}^s(\y)) -  \p(\hat{\Phi}^s(\y))\|_{\infty}~ds\\
    &\leq \int_{0}^{t} L \|\Phi^s(\y) - \hat{\Phi}^s(\y)\|_{\infty}~ds + 2 t \delta~.
\end{align*}
Defining $R(t) = \int_{0}^{t} L \|\Phi^s(\y) - \hat{\Phi}^s(\y)\|_{\infty}~ds + 2 t \delta$, we have
\begin{equation*}
    \dot{R}(t) = L \|\hat{\Phi}^t(\y) - \Phi^t(\y) \|_{\infty} + 2 \delta \leq L R(t) + 2 \delta~.
\end{equation*}
Letting $z(0) = R(0) = 0$ and $\dot z = Lz+2\delta \geq \dot R$, and solving for $z(t)$, we have
\begin{equation*}
    L \|\hat{\Phi}^t(\y) - \Phi^t(\y) \|_{\infty} + 2 \delta = \dot{R}(t) \leq z(t) = 2 \delta e^{L t}~.
\end{equation*}
Thus, for every $t \in [0,T]$ and $\y \in \Delta^n$, we have
\begin{equation}
    \label{eq:flow-approximation}
    \|\hat{\Phi}^t(\y) - \Phi^t(\y) \|_{\infty} \leq \frac{2 \delta}{L} \left(e^{L t} - 1\right) < \epsilon~,
\end{equation}
where we set $\epsilon = \frac{\delta}{L} \left(e^{L T} - 1\right)$.

We will now embed $\hat{\Phi}$, restricted to $\relint{\Delta^n}$, inside of a flow $\Gamma$ on $\reals^{n}_{++}$ given by a population system with generalized polynomial fitness functions.
Letting $\pi(\y) = (1 - \|\y\|_1)$, consider the population system $\textbf{M}$ on $\reals^{n}_{++}$ given by fitness functions $M_i(\y) = \pi(\y) + \frac{1}{y_i} p_i(\y)$ for each $i \in [n]$, where $p_i$ are the polynomials constructed above generating $\hat{\Phi}$.
Note that $\textbf{M}$ is given by the ODEs
\begin{equation*}
    \dot{y_i} = y_i \pi(\y) + p_i(\y)~.
\end{equation*}
By construction, $\relint (\Delta^n)$ is forward invariant under $\textbf{M}$, as $\pi(\y)=0$ on $\Delta^n$.
Furthermore, observe that for $\y = \y(t) \in \reals^n_{++}$ the population system $\textbf{M}$ has
\begin{align*}
    \frac{d}{d t} \|\y\|_1 &= \sum_{i \in [n]} y_i \pi(\y) + \sum_{i \in [n]} p_i(\y)\\
    &=  \|\y\|_1 \pi(\y) \\
    &= \|\y\|_1 (1 - \|\y\|_1)~,
\end{align*}
the logistic equation.
Thus, for every $\y \in \reals^n_{++}$, we know $\|\y\|_1 \to 1$ as $t \to \infty$.
It follows that $\relint (\Delta^n)$ is globally attracting for the dynamical system given by $\textbf{M}$.

As a final step define $\Gamma$ to be the flow on $\reals^n_{++}$ given by $\textbf{M}$.
By our construction, we know that a subspace of $\relint{\Delta^n}$ is a global attracting set of $\Gamma$ and that for $\y \in \relint{\Delta^n}$ we have $\Gamma = \hat{\Phi}$.  
All that remains to show is that $\Gamma$ is given by a system of GLV equations.
Recall that multivariate generalized polynomials on $\y \in \reals^n$ are defined as functions of the form
\begin{equation*}
    \sum_{j \in [m]} a_{j} \prod_{k \in [n]} y^{b_{k}}_k
\end{equation*}
where each $a_j \in \reals$ and $b_k \in \reals$. 
It is easy to check that the set of generalized polynomials is closed under multiplication and addition.
Therefore $\frac{1}{y_i} p_i(\y)$ is a generalized polynomial, $\pi(\y)$ is a generalized polynomial, and so $M_i(\y) = \pi(\y) + \frac{1}{y_i} p_i(\y)$ is a generalized polynomial for each $i \in [n]$.
Since the fitness functions $M_i$ for each $i \in [n]$ is given by a generalized polynomial, the flow $\Gamma$ on $\reals^n_{++}$ is given by a system of GLV equations by definition.
Furthermore, we showed that part (i) of the Theorem follows since $\Gamma$ has $\relint{\Delta^n}$ as a global attracting set.
In addition, we showed that part (ii) of the Theorem follows, as $\Gamma|_{\relint{\Delta^n}} = \hat\Phi|_{\relint{\Delta^n}}$.
\end{proof}

\subsection{Proof of Theorem~\ref{thm:GLV_to_Replicator}}\label{proof:glv_to_replicator}
\setcounter{theorem}{2}
\begin{theorem}
Let $\bar{\lambda} \in \reals^n$, $\bar{A} \in \reals^{n \times (m-1)}$, and $\bar{B} \in \reals^{(m-1) \times n}$ define a system of GLV equations~(eq.~\ref{eq:general_LV}) on $\reals^n_{++}$, where $m-1 \geq n$.
Let $\Gamma$ on $\reals^n_{++}$ be the flow generated by this system of GLV equations.
There exists a flow $\Theta$ on $\relint{\Delta^m}$ and a diffeomorphism $f:\reals^n_{++} \to \P \subseteq \relint{\Delta^m}$ such that:
\begin{enumerate}[label=(\roman*)]
    \item The flow $\Theta$ on $\relint{\Delta^m}$ is given by RD on a matrix game with payoff matrix $A \in \reals^{m \times m}$.
    \item The flow $\Theta|_{\P} = f(\Gamma)$ and $\Gamma = f^{-1}(\Theta|_{\P})$, where $\Theta|_{\P}$ is the flow given by $\Theta$ restricted to $\P$.
\end{enumerate}
\end{theorem}

\begin{proof}
Our proof proceeds by first embedding the GLV equations generating $\Gamma$ into a system of LV equations, and then constructing a diffeomorphism from the system of LV equations to a replicator system on a matrix game with payoff matrix $A \in \reals^{m \times m}$.
The embedding from the GLV equations into the LV equations ensures that the original system is easy to recover.
Our result follows immediately by composing the transformations from the given GLV equations all the way to the replicator system.
The first part of our proof uses an embedding trick introduced by~\cite{brenig1989universal}, whose properties (e.g.~smoothness) are explored by~\cite{hernandez1997lotka}.
The second part of our proof follows Theorem~$7.5.1$ by~\cite{hofbauer1998book}.

Consider the system of GLV equations on $\reals^n_{++}$ generating $\Gamma$
\begin{equation} \label{eq:og_GLV}
    \dot{x}_i = x_i \left(\bar{\lambda}_i +  \sum_{j \in [m-1]} \bar{A}_{ij} \prod_{k \in [n]} x^{\bar{B}_{jk}}_k \right), \qquad i \in [n]
\end{equation}
where $\bar{\lambda} \in \reals^n$, $\bar{A} \in \reals^{n \times (m-1)}$, $\bar{B} \in \reals^{(m-1) \times n}$, and $m-1 \geq n$.
Throughout this proof we will assume without loss of generality that $\bar{\lambda} = \0$, as we can simply append a column to $\bar{A}$ and add a row of zeros to $\bar{B}$.
In addition, we will also assume without loss of generality that $\bar{B}$ has column rank of $n$.\footnote{This assumption is without loss of generality because we can always add rows to $\bar{B}$ (i.e.~``increase $m$'') to ensure it has rank $n$.
It is important that we add a column of $0$'s to $A$ for every new row added to $\bar{B}$ and that any newly introduced species have unit valued initial conditions (i.e.~population of one at $t=0$).
We use the same trick when embedding eq.~\ref{eq:og_GLV} into the system given by eq.~\ref{eq:padded_GLV}.
Details about why this works are discussed below.}
We will embed the system given by eq.~\ref{eq:og_GLV} into a higher dimensional system of GLV equations by constructing matrices $\tilde{A},\tilde{B} \in \reals^{(m-1) \times (m-1)}$ as follows:
\begin{enumerate}[label=(\roman*)]
    \item The matrix $\tilde{A}$ has its first $n$ rows identical to $\bar{A}$ and its last $m-n-1$ rows as all zeros.
    That is, the matrix $\tilde{A} = \{\tilde{A}_{ij}\}_{i,j \in [m-1]}$ has $\tilde{A}_{ij} = \bar{A}_{ij}$ for $1 \leq i \leq n, j \in [m-1]$ and has $\tilde{A}_{ij} = 0$ for $n < i \leq m-1, j \in [m-1]$.
    \item The matrix $\tilde{B}$ has its first $n$ columns identical to $\bar{B}$ and its last $m-n-1$ columns set to any values which ensure $\tilde{B}$ is non-singular.
    That is, the matrix $\tilde{B} = \{\tilde{B}_{ij}\}_{i,j \in [m-1]}$ has $\tilde{B}_{ij} = \bar{B}_{ij}$ for $i \in [m-1], 1 \leq j \leq n$ and, for $i~\in~[m-1],~n~<~j~\leq~m-1$, has $\tilde{B}_{ij}$ set to any value ensuring the columns are linearly independent.
\end{enumerate}
These matrices define the GLV system on $\reals^{m-1}_{++}$ given by 
\begin{equation} \label{eq:padded_GLV}
    \dot{y}_i = y_i \left(\sum_{j \in [m-1]} \tilde{A}_{ij} \prod_{k \in [m-1]} y^{\tilde{B}_{jk}}_k \right), \qquad i \in [m-1]~.
\end{equation}

Observe that by construction $\dot{y}_i = 0$ for $i > n$ since $\tilde{A}_{ij}=0$ for $j \in [m-1]$, therefore we know that the dynamics of the $m-n-1$ newly introduced species are stationary.
Furthermore, since the newly introduced species have stationary dynamics, this construction ensures that the ODEs associated with the the original $n$ species only change by multiplying certain monomials with a constant--where for every species the multiplicative constant being introduced to the $j^{\text{th}}$ monomial is fully defined by the initial conditions of the newly introduced species and is given by the term $\prod_{k > n} y^{\tilde{B}_{jk}}_k$.
It follows that if we assign the initial condition $y_i = 1$ for $n < i \leq m-1$, then the ODEs $\dot{y}_i \equiv \dot{x}_i$ for $i \in [n]$.
As a consequence, this construction gives a natural embedding of the system given by eq.~\ref{eq:og_GLV} into the system given by eq.~\ref{eq:padded_GLV} while ensuring the dynamics of the first $n$ species remain identical.
We can formally write this embedding as the injective smooth map $\hat{f}_1:\reals^{n}_{++} \to \reals^{m-1}_{++}$, where $\hat{f}_1(\x) = (x_1, \cdots, x_n, 1,\cdots,1)$ ensures $\dot{y}_i \equiv \dot{x}_i$ for $i \in [n]$ and $\x \in \reals^{n}_{++}$.
It is worth noting that $\hat{f}_1$ is a diffeomorphism onto its image,\footnote{Readers familiar with differential topology might notice that $\hat{f}_1$ is an immersion, which implies $\hat{f}_1$ is a smooth embedding.} i.e.~it defines a diffeomorphism $f_1:\reals^{n}_{++} \to \hat{f}_1(\reals^{m-1}_{++})$.

Now, for $i \in [m-1]$, transform each $y_i$ in eq.~\ref{eq:padded_GLV} by
\begin{equation} \label{eq:quasimonomial_transform}
    y_i = \prod_{k \in [m-1]} z_k^{C_{ik}}, \qquad i \in [m-1]
\end{equation}
where $C \in \reals^{(m-1) \times (m-1)}$ is some non-singular matrix.
It was shown by~\cite{brenig1989universal} that transformations given by eq.~\ref{eq:quasimonomial_transform} define diffeomorphisms from $\reals^{m-1}_{++}$ to itself and that GLV equations are closed under these transformations.
In fact, this transformation maps eq.~\ref{eq:padded_GLV} to another system of GLV equations on $\reals^{m-1}_{++}$ given by
\begin{equation} \label{eq:transformed_padded_GLV}
    \dot{z}_i = z_i \left(\sum_{j \in [m-1]} \hat{A}_{ij} \prod_{k \in [m-1]} z^{\hat{B}_{jk}}_k \right), \qquad i \in [m-1]
\end{equation}
where $\hat{A}=C^{-1} \cdot \tilde{A}$ and $\hat{B}= \tilde{B} \cdot C$.
In particular, by using $C = \tilde{B}^{-1}$, the transformation given by eq.~\ref{eq:quasimonomial_transform} makes $\hat{B}=I$ (the identity matrix).
Therefore, by using $C = \tilde{B}^{-1}$, each generalized monomial in eq.~\ref{eq:transformed_padded_GLV} reduces to a single variable and we have the system of LV equations 
\begin{equation} \label{eq:padded_GLV_to_LV}
    \dot{z}_i = z_i \left(\sum_{j \in [m-1]} \hat{A}_{ij} z_j \right), \qquad i \in [m-1]~.
\end{equation}
Furthermore, by eq.~\ref{eq:quasimonomial_transform}, we have a diffeomorphism from $\y \in \reals^{m-1}_{++}$ to $\z \in \reals^{m-1}_{++}$ given by the transformations $z_i = \prod_{k \in [m-1]} y_k^{\tilde{B}_{ik}}$ for each $i \in [m-1]$.
Let $f_2:\reals^{m-1}_{++} \to \reals^{m-1}_{++}$ be this diffeomorphism from the GLV system given by eq.~\ref{eq:padded_GLV} to the LV system given by eq.~\ref{eq:padded_GLV_to_LV}.
By composing $\hat{f}_1$ with $f_2$ we have defined an embedding of our original GLV system to the LV system given by eq.~\ref{eq:padded_GLV_to_LV}.
In addition, since $\hat{f}_1$ ensures $y_i = 1$ for every $n < i \leq m-1$ in eq.~\ref{eq:padded_GLV}, we find that the embedding into the LV system can be written as $z_i = \prod_{k \in [n]} y_k^{\tilde{B}_{ik}} = \prod_{k \in [n]} x_k^{\bar{B}_{ik}}$ for each $i \in [m-1]$.

To conclude our construction, let $\p \in \relint{\Delta^{m}}$ be the mixed strategy of an agent playing an $m$-dimensional matrix game.
Furthermore, to make the notation of our argument easier to follow, add a homogenous compactifying dimension $z_{m} \equiv 1$ to the LV system from eq.~\ref{eq:padded_GLV_to_LV}.
That is, let $\z \in \reals^{m}_{++}$ be a population of species where $z_i$ is given by eq.~\ref{eq:padded_GLV_to_LV} for $i \in [m-1]$ and $z_{m} \equiv 1$.
In particular, we consider the system of LV equations given by the coefficient matrix $A \in \reals^{m \times m}$ where $A$ is simply the matrix $\hat{A}$ with an additional row and column of zeros (i.e.~$\{A_{l,h} = \hat{A}_{l,h}\}_{l,h \in [m-1]}$, $\{A_{l,m} = 0\}_{l \in [m]}$, and $\{A_{m,h} = 0\}_{h \in [m]}$).
Observe that, aside from the compactifying dimension $z_m$, this LV system is equivalent to the LV system given by eq.~\ref{eq:padded_GLV_to_LV}.
Now define a map $\z \to \p$, from populations in the LV system to mixed strategies in a game, by
\begin{equation} \label{eq:LV_to_strat}
    p_i = \frac{z_i}{\sum_{j \in [m]} z_j} , \qquad i \in [m]~.
\end{equation}
Similarly, define the inverse map by
\begin{equation} \label{eq:strat_to_LV}
    z_i = \frac{z_i}{z_{m}} = \frac{p_i}{p_{m}} , \qquad i \in [m]~.
\end{equation}
By the product rule, eq.~\ref{eq:LV_to_strat}, and eq.~\ref{eq:strat_to_LV} we have
\begin{align*}
    \dot{p}_i &= \frac{\dot{z}_i}{\sum_{j \in [m]} z_j} - \frac{z_i \sum_{j \in [m]} \dot{z_j}}{\left(\sum_{j \in [m]} z_j\right)^2}\\
    &= p_{m} \left(z_i \sum_{j \in [m]} A_{ij} z_j \right) - p^2_{m} z_i \left(\sum_{j \in [m]} z_j \sum_{k \in [m]} A_{jk} z_k \right) \\
    &= \frac{p_i}{p_{m}} \sum_{j \in [m]} A_{ij} p_j - \frac{p_i}{p_{m}} \left(\sum_{j \in [m]} p_j \sum_{k \in [m]} A_{jk} p_k \right) \\
    &= \frac{p_i}{p_{m}} \left(\sum_{j \in [m]} A_{ij} p_j - \sum_{j \in [m]} p_j \sum_{k \in [m]} A_{jk} p_k \right)
\end{align*}
for each $i \in [m]$.
By a change in velocity we can remove the term $\frac{1}{p_{m}}$. 
This yields
\begin{equation*}
    \dot{p}_i = p_i \left(\left(\sum_{j \in [m]} A_{ij} p_j \right) - \left(\sum_{j \in [m]} p_j \sum_{k \in [m]} A_{jk} p_k \right) \right) , \qquad i \in [m]~.
\end{equation*}
Noting that $\sum_{j \in [m]} A_{ij} p_j = (A \p)_i$ and $\sum_{j \in [m]} p_j \sum_{k \in [m]} A_{jk} p_k = \p^\Trans A \p$, we have derived the dynamical system
\begin{align}
    \dot{p}_i &= p_i \left((A \p)_i - \p^\Trans A \p \right), \qquad i \in [m] \label{eq:GLV_to_rep}
\end{align}
where eq.~\ref{eq:GLV_to_rep} is a replicator system (eq.~\ref{eq:replicator}) on the matrix game with payoff matrix $A \in \reals^{m \times m}$.
The converse direction, from eq.~\ref{eq:GLV_to_rep} to eq.~\ref{eq:padded_GLV_to_LV}, is derived in a similar way.\footnote{A full derivation of the inverse direction can be found in~\cite{hofbauer1998book} for Theorem $7.5.1$.}
We conclude that eq.~\ref{eq:LV_to_strat} is a diffeomorphism mapping trajectories of our LV system given by eq.~\ref{eq:padded_GLV_to_LV} onto trajectories of RD on $A$.
Since $z_m \equiv 1$, we can define the diffeomorphism $f_3: \reals^{m-1}_{++} \to \relint{\Delta^{m}}$ where $p_i = z_i/\left(1 + \sum_{j \in [m-1]} z_j\right)$ for $i \in [m]$ and $p_m = 1/\left(1 + \sum_{j \in [m-1]} z_j\right)$.

Taken as a whole, we have constructed an embedding from the original GLV system given by eq.~\ref{eq:og_GLV} to the replicator system on a matrix game given by eq.~\ref{eq:GLV_to_rep}.
The embedding itself can be written as the injective smooth map $\hat{f} : \reals^n_{++} \to \relint{\Delta^m}$ where $\hat{f} = \hat{f}_1 \circ f_2 \circ f_3$.
Furthermore, since we know $\hat{f}_1$ is diffeomorphic onto its image, we know there exists a diffeomorphism $f:\reals^n_{++} \to \P \subseteq \relint{\Delta^m}$ where $f = f_3 \circ f_2 \circ f_1$ and $\P = \hat{f}(\reals^n_{++})$.

Let $\Theta$ be the flow generated by eq.~\ref{eq:GLV_to_rep} and $\Theta|_{\P}$ be the flow given by $\Theta$ restricted to $\P$. 
Also, recall that $\Gamma$ was the flow generated by eq.~\ref{eq:og_GLV}.
From our derivations of $f_1$,$f_2$,and $f_3$, we know that $f(\Gamma) = \Theta|_{\P}$.
Furthermore, as diffeomorphisms as invertible and have $C^1$ inverses, we know $f^{-1}:\P \to \reals^n_{++}$ exists and that $f^{-1}(\Theta|_{\P}) = \Gamma$.
From our derivation of eq.~\ref{eq:GLV_to_rep} we know $\Theta$ is a flow on $\relint{\Delta^m}$ that is given by RD on a matrix game with the payoff matrix $A \in \reals^{m \times m}$ defined above.
Thus we have constructed a flow $\Theta$ and diffeomorphism $f$ satisfying properties (i) and (ii) in the Theorem, which concludes the proof.

Though not necessary for proving the Theorem, it is interesting to observe that the diffeomorphism $f$ can be written as 
\begin{equation} 
    p_i = \frac{z_i}{1 + \sum_{j \in [m-1]} z_j} = \frac{\prod_{k \in [n]} x_k^{\bar{B}_{ik}}}{1 + \sum_{j \in [m-1]} \prod_{k \in [n]} x_k^{\bar{B}_{jk}}}, \qquad i \in [m-1]~,
\end{equation}
and $p_m = 1/\left(1 + \sum_{j \in [m-1]} z_j\right) = 1/\left(1 + \sum_{j \in [m-1]} \prod_{k \in [n]} x_k^{\bar{B}_{jk}}\right)$.
Similarly, by composing the inverse directions of our construction, the inverse diffeomorphism $f^{-1}$ can be written as
\begin{equation} 
    x_i = y_i = \prod_{k \in [m-1]} z_k^{\tilde{B}^{-1}_{ik}} = \prod_{k \in [m-1]} \left(\frac{p_k}{p_{m}}\right)^{\tilde{B}^{-1}_{ik}}, \qquad i \in [n]~.
\end{equation}

\end{proof}

\subsection{The Lorenz Game: End-to-End Construction}\label{append:Lorenz}
In~\S\ref{sec:lorenz} we highlighted a construction of a matrix game that embeds the iconic Lorenz system under RD, but many of the details were omitted to keep the ideas concise and understandable.
In this Appendix we will go through the construction of this game in its entirety.
To start, note that the Lorenz system's strange attractor is given by the following three dimensional the system of ODEs in $\reals^3$
\begin{align*}
  \dot{x}_1 &= \sigma(x_2 - x_1) \\
  \dot{x}_3 &= x_1(\rho-x_3)-x_2 \\
  \dot{x}_3 &= x_1 x_2 - \beta x_3~,
\end{align*}
where $\sigma,\rho,\beta > 0$ are constants.
We will focus on the setting first studied by Lorenz himself when $\rho = 28$, $\sigma = 10$, and $\beta = 8/3$, but it is worth noting that this construction applied for any setting of these parameters.
Given these parameters, it is straightforward to show that there exists a spherical region with sufficiently large (constant) radius that is forward invariant under the Lorenz equations and is globally attracting.
To find such a region, define the ellipsoid region $\mathcal{E} = \{(x_1,x_2,x_3): \rho x_1^2 + \sigma x_2^2 + \sigma (x_3 - 2 \rho)^2 \leq c, c > 0 \}$ and choose $r>0$ such that $\mathcal{E}$ is contained inside a region bounded by the sphere $\mathcal{R} = \{(x_1,x_2,x_3): x_1^2 + x_2^2 + (x_3 - \rho - \sigma)^2 = r\}$; the region $\mathcal{R}$ is a globally attracting and forward invariant spherical region under the Lorenz system.

By shifting the solutions of the Lorenz equation by $r$ in the positive direction for all three dimensions, and then rearranging terms, we get a GLV system that is well defined on $\reals^3_{++}$ and can be written as

\vskip8pt
\noindent\hspace{5pt}\scalebox{0.94}{
\begin{math}
\begin{aligned}
  \dot{x}_1 &= \sigma \left((x_2-r) - (x_1-r)\right) \\
  \dot{x}_3 &= (x_1-r) \left(\rho - (x_3-r) \right)-(x_2-r) \\
  \dot{x}_3 &= (x_1-r) (x_2-r) - \beta (x_1-r)\\
\end{aligned}
\quad\Longrightarrow\quad
\begin{aligned}
  \dot{x}_1 &= x_1 \left(\sigma x_2 x_1^{-1} - \sigma \right) \\
  \dot{x}_2 &= x_2 \left(\eta x_1 x_2^{-1} - x_1 x_3 x_2^{-1} + r x_3 x_2^{-1} + \alpha x_2^{-1} - 1 \right) \\
  \dot{x}_3 &= x_3 \left(x_1 x_2 x_3^{-1} - r x_1 x_3^{-1} - r x_2 x_3^{-1} + \mu x_3^{-1} - \beta \right)
\end{aligned}
\end{math}
}
\vskip8pt

\noindent
where $\eta = \rho+r$, $\alpha = r - \rho r - r^2$, and $\mu = r^2 + \beta r$.
Furthermore, observe that this system of GLV equations is given by the matrices $\bar{A} \in \reals^{3 \times 10}$ and $\bar{B} \in \reals^{10 \times 3}$ which look as follows
\begin{equation*}
    \bar{A} = 
    \begin{bmatrix}
    \sigma & 0 & 0 & 0 & 0 & 0 & 0 & 0 & 0 & -\sigma\\
    0 & \eta & -1 & r & \alpha & 0 & 0 & 0 & 0 & -1\\
    0 & 0 & 0 & 0 & 0 & 1 & -r & -r & \mu & -\beta
    \end{bmatrix}~,
    \qquad
    \bar{B} =
    \begin{bmatrix}
    -1 & 1 & 0\\
    1 & -1 & 0\\
    1 & -1 & 1\\
    0 & -1 & 1\\
    0 & -1 & 0\\
    1 & 1 & -1\\
    1 & 0 & -1\\
    0 & 1 & -1\\
    0 & 0 & -1\\
    0 & 0 & 0
    \end{bmatrix}~.
\end{equation*}
Since we can rewrite the shifted Lorenz system in this GLV form, there is no need do derive the approximation highlighted in Theorem~\ref{thm:GLV_approx_of_general_sys} and we can directly apply Theorem~\ref{thm:GLV_to_Replicator}.

Using the embedding trick by~\cite{brenig1989universal} that is explained in Appendix~\ref{proof:glv_to_replicator}, we first embed this GLV into a higher dimensional GLV system on $\reals^{10}_{++}$ given by the matrices
\begin{equation*}
    \tilde{A} = 
    \begin{bmatrix}
    \sigma & 0 & 0 & 0 & 0 & 0 & 0 & 0 & 0 & -\sigma\\
    0 & \eta & -1 & r & \alpha & 0 & 0 & 0 & 0 & -1\\
    0 & 0 & 0 & 0 & 0 & 1 & -r & -r & \mu & -\beta\\
    0 & 0 & 0 & 0 & 0 & 0 & 0 & 0 & 0 & 0\\
    0 & 0 & 0 & 0 & 0 & 0 & 0 & 0 & 0 & 0\\
    0 & 0 & 0 & 0 & 0 & 0 & 0 & 0 & 0 & 0\\
    0 & 0 & 0 & 0 & 0 & 0 & 0 & 0 & 0 & 0\\
    0 & 0 & 0 & 0 & 0 & 0 & 0 & 0 & 0 & 0\\
    0 & 0 & 0 & 0 & 0 & 0 & 0 & 0 & 0 & 0\\
    0 & 0 & 0 & 0 & 0 & 0 & 0 & 0 & 0 & 0\\
    \end{bmatrix}~,
    \enskip
    \tilde{B} =
    \begin{bmatrix}
    -1 & 1 & 0 & 0 & 0 & 0 & 0 & 0 & 0 & 0\\
    1 & -1 & 0 & 1 & 0 & 0 & 0 & 0 & 0 & 0\\
    1 & -1 & 1 & 0 & 0 & 0 & 0 & 0 & 0 & 0\\
    0 & -1 & 1 & 0 & 0 & 0 & 0 & 0 & 0 & 0\\
    0 & -1 & 0 & 0 & 1 & 0 & 0 & 0 & 0 & 0\\
    1 & 1 & -1 & 0 & 0 & 1 & 0 & 0 & 0 & 0\\
    1 & 0 & -1 & 0 & 0 & 0 & 1 & 0 & 0 & 0\\
    0 & 1 & -1 & 0 & 0 & 0 & 0 & 1 & 0 & 0\\
    0 & 0 & -1 & 0 & 0 & 0 & 0 & 0 & 1 & 0\\
    0 & 0 & 0 & 0 & 0 & 0 & 0 & 0 & 0 & 1
    \end{bmatrix}~.
\end{equation*}
Next we use the diffeomorphism given by eq.~\ref{eq:quasimonomial_transform} to transform this higher dimensional GLV system into the LV system on $\reals^{10}_{++}$ given by the matrix
\begin{equation*}
    \hat{A} = 
    \begin{bmatrix}
    -\sigma & \eta & -1 & r & \alpha & 0 & 0 & 0 & 0 & (\sigma - 1)\\
    \sigma & -\eta & 1 & -r & -\alpha & 0 & 0 & 0 & 0 & (1 - \sigma)\\
    \sigma & -\eta & 1 & -r & -\alpha & 1 & -r & -r & \mu & (1 - \sigma - \beta)\\
    0 & -\eta & 1 & -r & -\alpha & 1 & -r & -r & \mu & (1 - \beta)\\
    0 & -\eta & 1 & -r & -\alpha & 0 & 0 & 0 & 0 & 1\\
    \sigma & \eta & -1 & r & \alpha & -1 & r & r & -\mu & (\beta - \sigma - 1)\\
    \sigma & 0 & 0 & 0 & 0 & -1 & r & r & -\mu & (\beta - \sigma)\\
    0 & \eta & -1 & r & \alpha & -1 & r & r & -\mu & (\beta - 1)\\
    0 & 0 & 0 & 0 & 0 & -1 & r & r & -\mu & \beta\\
    0 & 0 & 0 & 0 & 0 & 0 & 0 & 0 & 0 & 0
    \end{bmatrix}~.
\end{equation*}
From the embedding trick we know that each the states $\z \in \reals^{10}_{++}$ of this LV system are given by $z_i = \prod_{j \in [3]} x_j^{\bar{B}_{ij}}$, where $\x \in \reals^3_{++}$ is from the shifted Lorenz system.
As such, notice that each row of $\hat{A}$ is associated with a monomial in the shifted Lorenz system.

To conclude the construction, we apply the diffeomorphism by~\cite{hofbauer1998book} to this LV system and get game matrix $A \in \reals^{11 \times 11}_{++}$ that can be written as 
\begin{equation*}
    A = 
    \begin{bmatrix}
    -\sigma & \eta & -1 & r & \alpha & 0 & 0 & 0 & 0 & (\sigma - 1) & 0\\
    \sigma & -\eta & 1 & -r & -\alpha & 0 & 0 & 0 & 0 & (1 - \sigma) & 0\\
    \sigma & -\eta & 1 & -r & -\alpha & 1 & -r & -r & \mu & (1 - \sigma - \beta) & 0\\
    0 & -\eta & 1 & -r & -\alpha & 1 & -r & -r & \mu & (1 - \beta) & 0\\
    0 & -\eta & 1 & -r & -\alpha & 0 & 0 & 0 & 0 & 1 & 0\\
    \sigma & \eta & -1 & r & \alpha & -1 & r & r & -\mu & (\beta - \sigma - 1) & 0\\
    \sigma & 0 & 0 & 0 & 0 & -1 & r & r & -\mu & (\beta - \sigma) & 0\\
    0 & \eta & -1 & r & \alpha & -1 & r & r & -\mu & (\beta - 1) & 0\\
    0 & 0 & 0 & 0 & 0 & -1 & r & r & -\mu & \beta & 0\\
    0 & 0 & 0 & 0 & 0 & 0 & 0 & 0 & 0 & 0 & 0 \\
    0 & 0 & 0 & 0 & 0 & 0 & 0 & 0 & 0 & 0 & 0
    \end{bmatrix}~.
\end{equation*}
The solution of RD on $A$ is plotted in Figure~\ref{fig:lorenz}.
It is worth noting that the last row and column are all zeros since they correspond to the compactifying dimension added during the diffeomorphism by~\cite{hofbauer1998book}.
Similarly, the second to last row of zeros and the corresponding column are from the matrix $\hat{A}$ and serve the role of keeping track of the constants in the shifted Lorenz system.
In addition, we have a diffeomorphism $f: \reals^3_{++} \to \P \subset \relint{\Delta^{11}}$ from $\x \in \reals^3_{++}$ to $\p \in \relint{\Delta^{11}}$ that is given by
\begin{equation} 
    p_i = \frac{z_i}{1 + \sum_{j \in [m-1]} z_j} = \frac{\prod_{k \in [n]} x_k^{\bar{B}_{ik}}}{1 + \sum_{j \in [m-1]} \prod_{k \in [n]} x_k^{\bar{B}_{jk}}}, \qquad i \in [10]~,
\end{equation}
and $p_{11} = 1/\left(1 + \sum_{j \in [m-1]} z_j\right) = 1/\left(1 + \sum_{j \in [m-1]} \prod_{k \in [n]} x_k^{\bar{B}_{jk}}\right)$.
By finding that 
\begin{equation*}
    \tilde{B}^{-1} = 
    \begin{bmatrix}
    0 & 0 & 1 & -1 & 0 & 0 & 0 & 0 & 0 & 0\\
    1 & 0 & 1 & -1 & 0 & 0 & 0 & 0 & 0 & 0\\
    1 & 0 & 1 & 0 & 0 & 0 & 0 & 0 & 0 & 0\\
    1 & 1 & 0 & 0 & 0 & 0 & 0 & 0 & 0 & 0\\
    1 & 0 & 1 & -1 & 1 & 0 & 0 & 0 & 0 & 0\\
    0 & 0 & -1 & 2 & 0 & 1 & 0 & 0 & 0 & 0\\
    1 & 0 & 0 & 1 & 0 & 0 & 1 & 0 & 0 & 0\\
    0 & 0 & 0 & 1 & 0 & 0 & 0 & 1 & 0 & 0\\
    1 & 0 & 1 & 0 & 0 & 0 & 0 & 0 & 1 & 0\\
    0 & 0 & 0 & 0 & 0 & 0 & 0 & 0 & 0 & 1
    \end{bmatrix}~,
\end{equation*}
we find that the inverse diffeomorphism $f^{-1}$ can be written as
\begin{equation} 
    x_i = y_i = \prod_{k \in [m-1]} z_k^{\tilde{B}^{-1}_{ik}} = \prod_{k \in [m-1]} \left(\frac{p_k}{p_{m}}\right)^{\tilde{B}^{-1}_{ik}}, \qquad i \in [3]~.
\end{equation}

\end{document}